\documentclass{article}

\usepackage{microtype}
\usepackage{graphicx}
\usepackage{subfigure}
\usepackage{booktabs} 
\usepackage{bbm}
\usepackage{hyperref}

\usepackage[accepted]{icml2024_arxiv}

\usepackage{amsmath}
\usepackage{amssymb}
\usepackage{mathtools}
\usepackage{amsthm}

\usepackage[capitalize,noabbrev]{cleveref}

\theoremstyle{plain}
\newtheorem{theorem}{Theorem}[section]

\newtheorem{lemma}[theorem]{Lemma}

\theoremstyle{definition}

\theoremstyle{remark}

\usepackage[textsize=tiny]{todonotes}
\icmltitlerunning{Preprint}

\begin{document}

\twocolumn[
\icmltitle{Sobolev Training for  Operator Learning}
\icmlsetsymbol{equal}{*}
\begin{icmlauthorlist}
\icmlauthor{Namkyeong Cho}{equal,postech}
\icmlauthor{JUNSEUNG RYU}{equal,postech}
\icmlauthor{Hyung Ju Hwang}{postech}
% % \icmlauthor{Firstname4 Lastname4}{sch}
% % \icmlauthor{Firstname5 Lastname5}{yyy}
% % \icmlauthor{Firstname6 Lastname6}{sch,yyy,comp}
% % \icmlauthor{Firstname7 Lastname7}{comp}
% % \icmlauthor{Firstname8 Lastname8}{sch}
% % \icmlauthor{Firstname8 Lastname8}{yyy,comp}
\end{icmlauthorlist}
% \icmlauthor{Namkyeong Cho}{postech}
% % \icmlaffiliation{affiliation}{Affiliation Name}
% \icmlemail{name@domain.com}
% \icmlauthor{Name}{affiliation}
% \icmlaffiliation{affiliation}{Affiliation Name}
% \icmlemail{name@domain.com}
\icmlaffiliation{postech}{Department of Mathematics
Pohang University of Science and Technology Pohang, 37673, Republic of Korea}
% \icmlaffiliation{am2}{Company Name, Location, Country}
% % \icmlaffiliation{sch}{School of ZZZ, Institute of WWW, Location, Country}
\icmlcorrespondingauthor{Hyung Ju Hwang}{hjhwang@postech.ac.kr}
% \icmlcorrespondingauthor{JUNSEUNG RYU}
% {first2.last2@www.uk}
\icmlkeywords{}
\vskip 0.3in
]



% \printAffiliationsAndNotice{\icmlEqualContribution} % otherwise use the

%RYU's own commands
\newcommand{\bu} {\mathbf{u}}
\newcommand{\bx} {\mathbf{x}}
\newcommand{\bg} {\mathbf{g}}

%Main Text
\printAffiliationsAndNotice{\icmlEqualContribution} 
\begin{abstract}
   This study investigates the impact of Sobolev Training on operator learning frameworks for improving model performance. Our research reveals that integrating derivative information into the loss function enhances the training process, and we propose a novel framework to approximate derivatives on irregular meshes in operator learning. Our findings are supported by both experimental evidence and theoretical analysis.
   This demonstrates the effectiveness of Sobolev Training in approximating the solution operators between infinite-dimensional spaces.
\end{abstract}

\section{Introduction}\label{introduction}
Over the last few decades, machine learning has seen substantial advancements in various areas, such as computer vision \cite{dosovitskiy2020image,niemeyer2021giraffe}, natural language processing \cite{vaswani2017attention,radford2018improving,lewis2019bart,kenton2019bert}, and physical modeling \cite{RAISSI2019686,jagtap2020extended}. In physical modeling, machine learning methods are valuable tools for solving complex partial differential equations (PDEs). An emerging field within this domain is operator learning, which uses neural networks to learn mappings between infinite-dimensional spaces.

By utilizing the expressive power of neural networks, neural operators can be trained using datasets of input and output functions. Once trained, these operators can predict unseen equations with just a single forward computation. Compared to classical FDM/FEM methods, these ML-based methods greatly accelerate the equation-solving process.
In the field of operator learning, FNO \cite{li2021burigede} and DeepONet \cite{lu2021learning} are the most popular models. For a comprehensive overview, we refer to \cite{kovachki2021neural, hao2022physics, zappala2021operator} and the references therein.

Sobolev Training, introduced in \cite{czarnecki2017sobolev}, enhances traditional training methods by integrating the derivatives of the target function along with its values. This method improves the accuracy of predictions and boosts data efficiency and the model's generalization ability in function approximation. The authors of \cite{czarnecki2017sobolev} provide theoretical and empirical validation for these improvements.

A recent and significant application of Sobolev Training can be seen in Physics-Informed Neural Networks (PINNs). Sobolev-PINNs, introduced in \cite{son2021sobolev}, represent an adaptation of Sobolev Training. This is achieved by designing a novel loss function that directs the neural network to minimize errors within the corresponding Sobolev space. Empirical evidence has demonstrated that Sobolev-PINNs accelerate the convergence rate in solving complex PDEs. Furthermore, the authors of \cite{son2021sobolev} provide theoretical justifications for guaranteeing convergence in several equations.

Motivated by the previously mentioned works, we propose the Sobolev Training method for operator learning and provide theoretical and experimental support for its efficacy.

Here are the two significant challenges in our work. First, there is a need to approximate derivatives on irregular meshes to apply the Sobolev Training method to various existing models, as most recent results involve irregular meshes. 
To address this challenge, we adopt a local approximation method using moving least squares (MLS) within a coordinate system that is locally constructed. This system is derived from local principal component analysis (PCA) utilizing K-nearest neighbors (KNN), as developed in \cite{lipman2006error, liang2013solving}.

Second, the solution operator of PDEs exhibits non-local properties, which can be described using an integral operator. This approach requires additional work in theoretical analysis. Motivated by the work of \cite{tian2017analytical}, we provide a convergence analysis for a single layer with ReLU activation when the target function is described using an integral operation. 
We prove that including derivatives in the loss function improves convergence speed. To the best of our knowledge, this is the first work to present a convergence analysis in the context of operator learning, even without Sobolev Training. Furthermore, we show empirical evidence that PCgrad, introduced in \cite{yu2020gradient}, works well with the Sobolev Training method and enhances the performance of existing models.

Our contributions in this work are summarized as follows:
\begin{itemize}
    \item Integration of a derivative approximation algorithm with existing operator learning models.
    
    \item Presentation of the first theoretical analysis on convergence in the field of operator learning.
    
    \item Introduction of Sobolev Training within operator learning, supported by theoretical and empirical evidence.
\end{itemize}

\section{Preliminaries}
\subsection{Notations}\label{subsec:notation}
Let us begin with the notations that are used throughout this paper.
Given a multi-index $\alpha=(\alpha_1, \cdots, \alpha_n)\in (\mathbb{N}\cup \{ 0  \} )^{n}$ and a vector $x=(x^1, \cdots, x^n)\in \mathbb{R}^n$, we use $x^\alpha$ to denote
$$
x^\alpha = (x^1)^{\alpha_1}(x^2)^{\alpha_2}\cdots (x^n)^{\alpha_n}.
$$
We define $|\alpha|=\alpha_1+\alpha_2+\cdots +\alpha_n$. Let $\Omega$ be a bounded domain in $\mathbb{R}^n$. The space of continuously differentiable functions up to order $M \in \mathbb{N}$ on $\Omega$ is denoted by $C^M(\Omega)$.
\subsection{Sobolev spaces and Sobolev Training}
Sobolev spaces are fundamental 
function spaces that appear in analyzing PDEs and functional analysis.
For any $\varphi \in C^{|\alpha|}(\Omega)$, we denote  
$$
D^{\alpha}_x\varphi  = \frac{\partial^{\alpha_1}}{\partial x^1} \cdots
\frac{\partial^{\alpha_n}}{\partial x^n}\varphi.
$$ 
For a given  $u\in L^p(\Omega)$, if there exists $w\in L^p(\Omega)$ satisfyng  
$$
\int_{\Omega} u D_x^{\alpha} \xi  \,dx
= (-1)^{|\alpha|}
\int_{\Omega} w\xi\,dx \quad \forall \varphi \in C^\infty_0(\Omega), 
$$
then we refer $w$ as the $\alpha^{th}$-weak derivative of $u$ and denote  
$w = D_x^{\alpha} u$. Unless otherwise specified, we use weak derivatives throughout this paper. We refer to Chapter 5 of  \cite{evans2022partial} for detailed information regarding the Sobolev spaces.
For $k\in \mathbb{N}$ and $1\leq p<\infty $, we define a Sobole space $W^{k,p}(\Omega)$ by
$$
W^{k, p}(\Omega):=
\{u \in L^{p}(\Omega) :|D^{\alpha}_{x} u |\in L^p(\Omega) \ \ \forall |\alpha| \leq  k \}.
$$
The idea of training using Sobolev spaces is introduced in  \cite{czarnecki2017sobolev}. For a given function $f:\mathbb{R}^n \to \mathbb{R}^m$ with training set consisting of pairs ${(x_j,f(x_j))}_{j=1}^{J}$, training a neural network $g_{\theta}(x)$ consists of finding $\theta$ that minimizes
$$
\sum_{j=1}^{J} l(g_\theta(x_j), f(x_j)),
$$
for some given loss function  $l$. In Sobolev Training, the loss function above is replaced by 
$$
\sum_{j=1}^{J}
\left[l(g_\theta(x_j), f(x_j)) + \sum_{i=1}^{n}l_i(D_x^{e_i}g_\theta(x_j), D_x^{e_i} f(x_j))\right],
$$
where $D_x^{e_i}$ is the derivative of function in $i^{th}$ variable and 
 $l_i$ is a loss function measuring error for  each $i\in[1, \cdots, n]$. 

\subsection{Operator Learning} 
Operator learning is an emerging area of study in machine learning that employs neural networks to train solution operators for various equations. The DeepONet \cite{lu2021learning} and FNO \cite{li2021burigede} are considered fundamental models in the field of operator learning. FNO efficiently learns complex PDEs using Fourier analysis. 
At the same time, DeepONet is a model developed based on the Universal Approximation Theorem for operators that can handle many problems, including multiple input functions and complex boundary conditions. 
However, like other data-driven methods, DeepONet requires substantial data. Many efforts have been made to improve existing models. 
Let us introduce some of the notable developments in several studies. For example, Geo-FNO \cite{li2022fourier} is developed to
expand the FNO model to solve equations with irregular meshes in FNO. 
The transformer architecture, introduced in \cite{vaswani2017attention}, has achieved significant success in various fields such as natural language processing \cite{kenton2019bert, radford2018improving} and computer vision \cite{dosovitskiy2020image}. 
In \cite{cao2021choose}, the transformer is applied to operator learning, emphasizing the importance of selecting the appropriate transformer model. 
Recently, authors of  \cite{hao2023gnot} apply advanced machine learning techniques to transformer-based models. This approach enables handling diverse and complex datasets, resulting in more robust and scalable models.

%%%%%%%%%%%%%%%%%%%%%%%%%%%%%%%%%%%%%%%%%%%%%%%%%%%%%%
\subsection{Results on the convergence analysis}  
The convergence analysis of gradient descent has been rigorously studied in deep learning research. Convergence studies, particularly for linear models, serve as fundamental building blocks in understanding the behavior of gradient descent in complex architectures. 
In the notable works \cite{arora2018convergence, bartlett2018gradient}, authors provide theoretical insights into linear models. Parallel to these, authors of  \cite{liu2020improved} investigate stochastic gradient methods with momentum (SGDM) and theoretically show improved convergence results compared to the usual stochastic gradient method. The convergence dynamics in two-layer neural networks with ReLU activation are explored extensively in  \cite{tian2017analytical, li2017convergence}. In particular, \cite{tian2017analytical} explores population gradients, where the inputs are drawn from zero-mean spherical Gaussian distributions, and establishes interesting theoretical results, including the properties of critical points and the probability of convergence to the optimal weights with randomly initialized parameters. 
Such comprehensive studies collectively enhance our understanding of the intricate convergence behaviors in more complicated and larger neural network structures.

\section{Proposed method}\label{sec:proposed_method}
This section introduces our method and theoretical result for the Sobolev Training for operator learning.

\subsection{Approximating derivatives in non-uniform meshes}\label{subsec:derivative_approx}
We adopt the reconstructing local surface approximation method using the moving least-square (MLS) method developed in  \cite{liang2013solving, lipman2006error}.

Let us assume that a function $u:\Omega \subset \mathbb{R}^n \to \mathbb{R}$ and a irregular mesh $\{x_j\}_{j=1}^{J} \subset \Omega $ are given.
We denote by $\Pi_{m}^{n}$
  the set of 
$n$-dimensional polynomial functions with a maximum degree less than or equal to 
$m$. For now, let us fix $\bar{x} \in \{x_j\}_{j=1}^{J}$. We then  find $K$-nearest points of $\bar{x}$ denoted by $\{\bar{x}_{1}, \bar{x}_{2}, \cdots, \bar{x}_{K}\} \subset \{x_j\}_{j=1}^{J}$. We intend to solve the following minimization problem
\begin{equation}\label{eq:minimize}
\min_{p \in \Pi_{m}^{n}} 
\sum_{k=1}^{K}\omega(\|\bar{x}_{k}-\bar{x}\|) \|p(\bar{x}_{k}) - u(\bar{x})\|^2,
\end{equation}
where $\omega(\cdot)$ is some weight function to be determined. In this work, we use the  \cref{weight_def}. 
It is directly checked that the number of the biasis $\Pi_{m}^{n}$  is $I := \frac{(m+n)!}{m!n!}$.
Let $b(x)$ represent a polynomial basis vector, with each element of $b(x)$  denoted by \( (x-\bar{x})^\alpha \) for any multi-index $\alpha$ satisfying $ |\alpha| \leq m$. 
We denote $c\in \mathbb{R}^{I}$ by the coefficient vector of dimension $I$, and $c_\alpha$ represents the coefficient corresponding to $(x-\bar{x})^\alpha$. 
Each element in $\Pi_{m}^{n}$ can be expressed as 
\begin{equation}\label{eq:polynomial}
p(x) = c^T b(x) = c \cdot b(x) 
= \sum_{|\alpha|\leq m} c_\alpha (x-\bar{x})^\alpha.
\end{equation}
By taking partial derivative in \cref{eq:minimize}, the coefficients can be determined by
\begin{equation}\label{eq:coefficient}
    \begin{aligned}
&c=E^{-1}\sum_{k=1}^{K}\omega(\|\bar{x}_{k}-\bar{x}\|)b(\bar{x}_{k})u(\bar{x}_{k}), \\
    &\text{with } E =\left(\sum_{k=1}^{K}\omega(\|\bar{x}_{k}-\bar{x}\|)b(\bar{x}_{k})b(\bar{x}_{k})^T\right).
    \end{aligned}
\end{equation}
Now, consider a set of data points $\{x_j, u(x_j)\}_{j=1}^{J}$. To solve the minimizing problem described in  \cref{eq:minimize}, we apply  \cref{eq:coefficient}, substituting $\bar{x}$ with each $x_j$ in the dataset. Subsequently, the approximation of $D_{x}^{\alpha}(x_j) u$ is represented by $c_{j, \alpha}$, which is the coefficient of $(x-x_j)^\alpha$ for each polynomial. We summarize overall in Algorithm \ref{alg:gradient_est}.

\begin{algorithm}[tb]
   \caption{Estimating derivatives in non-uniform meshes}
   \label{alg:gradient_est}
\begin{algorithmic}
    \STATE {\bfseries Define:} K: \text{number of nearest points}
    \STATE \phantom{{\bfseries Define:}} 
    m: \text{polynomial order degree}
    % \STATE \phantom{{\bfseries Define:}} 
   \STATE {\bfseries Input:} $\{(x_j, u(x_j)\}_{j=1}^{J}, K, m$ 
   \STATE Find $K$ nearest points for each  $x_j$:\\ 
   \quad  $x_{KNN, j }:=(x_{KNN, j , 1}, \cdots,  x_{KNN, j,  K})$.\\  
   \STATE Define a polynomial $p_j(x)$  for each $j$ defined by  \cref{eq:polynomial}
   \STATE Calculate  coefficient vector $c_j\in \mathbb{R}^I$ using \cref{eq:coefficient}
   \STATE {\bfseries Return:} coefficient vector $c_j$ for each $i$.
\end{algorithmic}
\end{algorithm}

In the following lemma, we prove that the result of Algorithm \ref{alg:gradient_est} indeed approximates the derivative of the target function in the Sobolev space. 
To this end, we assume that the data points $\{x_j\}_{j=1}^{J}$ are uniformly sampled so that
\begin{equation}\label{eq:approx_}
\frac{1}{J} \sum_{j=1}^{J} |f(\tilde{x}_i)|^2 \stackrel{C_0}{\approx}\frac{1}{|\Omega|}\int_{\Omega} |f|^2 \,dx ,
\end{equation}
where $C_0>0$ is some universal constant. In this context, for each $j$, $\tilde{x}_j$ denotes any point near $x_j$, satisfying the condition $\|x_j - \tilde{x}_j\| \leq \varepsilon$, where $\varepsilon > 0$ is sufficiently small. For more details on the Monte-Carlo integration method, refer to \cite{press2007numerical, newman1999monte}.

\begin{lemma}\label{lem:grad_convergence}
Suppose $u \in W^{M,2}(D)$ for some $M \in \mathbb{N}$, and consider a set of grid points $\{x_j\}_{j=1}^{J}$  satisfing the Monte-Carlo approximation assumption \cref{eq:approx_} for all $D^\alpha_x u$ with $|\alpha| \leq M$ and for some $C_0 > 0$. Define 
$$
h = \min_{j,k} \|x_j - x_{\text{KNN}, j,k}\|,
$$
and let $c_j$ be the coefficient vector obtained from Algorithm \ref{alg:gradient_est}. 
Then, for all $m \leq M$, there exists a constant $C = C(C_0, N, K, m, M, |\Omega|)$ such that the following inequality holds:
\[
\frac{1}{J}
\sum_{j=1}^{J}\sum_{|\alpha| = m} \|c_{j,\alpha} - D^{\alpha} u(x_j)\|^2 \leq C h^{M-m} \|u\|_{W^{M,2}(\Omega)}.
\]
In this context, we use $c_{j, \alpha}$ to denote the element of $c_j$ from Algorithm \ref{alg:gradient_est}, corresponding to the $(x - x_j)^\alpha \in \Pi_{m}^{n}$. 
\end{lemma}

The proof can be found in  Appendix \ref{appendix:prove_derivative_lem}
%%%%%%%%%%%%%%%%%%%%%%%%%%%%%%%%%%%%%%%%
\subsection{Convergence analysis for the Sobolev  Training for operator learning}\label{subsec:convergence_analysis}
This subsection provides theoretical evidence of the effectiveness of Sobolev Training in operator learning.
Suppose that the query points $\{x_j\}_{j=1}^{J}$ is given from spherical Gaussian distribution, $\mathcal{N}(0,I)$.
Suppose $\{u_k(x), v_k(x)\}_{k=1}^{N}$ is a pair of input and output real-valued functions. 
To capture non-local properties of the solution operator of PDEs, the solution can be described as an integral operator.
We further assume that there exists $\{y_l\}_{l=1}^{\tilde{J}}$ such that  
\begin{equation}\label{eq:u_represent}
\begin{aligned}
u_k(x) &=\frac{1}{|\Omega|} \int_{\Omega} \kappa(x,y) v_k(y) \,dy\\
&\approx \frac{1}{\tilde{J}} \sum_{l=1}^{\tilde{J}} \kappa(x,y_l) v_k(y_l)
\end{aligned}
\end{equation}
for some kernel function $\kappa(x,y)$.
The kernel function $\kappa(x,y)$ can be represented as 
\begin{equation}\label{eq:kappa_multiply}
\kappa(x,y) = \sum_{\tilde{i}=1}^{\tilde{I}} \varphi^{\tilde{i}}(x)\psi^{\tilde{i}}(y),
\end{equation}
where each $\varphi^{\tilde{i}}$ and $\psi^{\tilde{i}}$ is a neural network. 
For simplicity, we assume that   
\begin{equation}\label{eq:kappa_assume}
\tilde{I}=1 \quad \text{and} \quad
    \varphi(x)=\psi(x)= \sigma(w^Tx)
\end{equation} for some parameter $w \in \mathbb{R}^{n}$ and the activation function, $\sigma$ is the ReLu function.
By the Universal Approximation Theorem \cite{lu2017expressive}, a two-layer network can approximate each $v_k\in W^{1,2}(\Omega)\subset L^2(\Omega)$. For the analytic simplicity,  we suppose that 
\begin{equation}\label{v_assume}
v_k(x) = \mu_k\sigma((w^*)^Tx)\quad \text{for some }w^*\in \mathbb{R}^n.
\end{equation}

From these assumptions above, we have
$$
u_k(x;w) =
\frac{1}{\tilde{J}}
\sum_{l=1}^{\tilde{J}} \mu_k \sigma(w^Tx)\sigma(w^Ty_l)\sigma((w^*)^Tx).
$$
By denoting $X = [x_1, \cdots, x_J] \in \mathbb{R}^{J\times n}$, the usual $L^2$ loss function can be expressed as
\begin{equation}\label{L2_loss}
    \mathcal{L}_{L^2} = \frac{1}{N}\sum_{k=1}^{N} 
 \left| u_k(X;w)
 - 
 u_k(X;w^*) \right|^2.
\end{equation}
For the Sobolev Training, we  denote
\begin{equation}\label{L2_der_loss}
    \mathcal{L}_{Der} = \frac{1}{N}\sum_{k=1}^{N} 
 \left| D_x u_k(X;w)
 - 
 D_x u_k(X;w^*) \right|^2
\end{equation}
and use the loss function
% \begin{equation}\label{sob_loss}
% \begin{aligned}
%     \mathcal{L}_{Sob} &= \frac{1}{N}\sum_{k=1}^{N} 
%  \left| u_k(X;w)
%  - 
%  u_k(X;w^*) \right|^2 \\
%  &\quad 
%  +\frac{1}{N}\sum_{k=1}^{N} 
%  \left| \nabla_x u_k(X;w)
%  - 
%  \nabla_x u_k(X;w^*) \right|^2 .
%  \end{aligned}
% \end{equation}
\begin{equation}\label{sob_loss}
\begin{aligned}
    \mathcal{L}_{Sob} &= \mathcal{L}_{L^2}+ \mathcal{L}_{Der}.
 \end{aligned}
\end{equation}
For each training method given the initial parameter $w_0 \in \mathbb{R}^n$, the parameters $w$ are updated via the following gradient descent  method for some $t>0$
\begin{equation}\label{eq:gradient_descent}
\begin{aligned}
w_{L^2} &= w_0 - t \mathbb{E}_{X}\left[\frac{\partial \mathcal{L}_{L^2}}{\partial w}\right], \\
w_{Sob} &= w_0 - t\mathbb{E}_{X}\left[\frac{\partial \mathcal{L}_{Sob}}{\partial w}\right].
\end{aligned}
\end{equation}
By taking $t\to 0$ in \cref{eq:gradient_descent}, we have the following gradient flow:
\begin{equation}\label{eq:gradient_flow}
\begin{aligned}
    \frac{dw_{L^2}}{dt} &=  -\mathbb{E}_{X}\left[\frac{\partial \mathcal{L}_{L^2}}{\partial w}\right], \\
     \frac{dw_{Sob}}{dt} &= -\mathbb{E}_{X}\left[\frac{\partial \mathcal{L}_{Sob}}{\partial w}\right].
    \end{aligned}
\end{equation}

Now, we are ready to present our main theorem.
\begin{theorem}\label{lem:convergence_analysis}
Suppose that we are under the assumption \cref{eq:u_represent,eq:kappa_multiply,eq:kappa_assume,v_assume} and the activation function, $\sigma(x)$ is $ReLU(x)=\max(0,x)$. 
Also, let $w_{L^2}$ and $w_{Sob}$ be parameter vectors defined as \cref{eq:gradient_flow}.
Then, we have the  following results:
\begin{itemize}
    \item If 
    $w_0 := w(0) \in \{w:\|w_{L^2}-w^{*}\|<\|w^*\|\}$, then $$
    \frac{d}{dt} \|w_{L^2}-w^{*}
\|^2 <0.
    $$ 
    Moreover, we have the convergence results 
    $w_{L^2} \to w^*$ as $t \to\infty$.

    \item 
Under the same assumption as in the previous statement, we have
\begin{align*}
&\frac{d}{dt} 
 \|w_{Sob} -w^* \|^2\leq 
\frac{d}{dt} 
\|w_{L^2} -w^* \|^2< 0.
\end{align*}
In particular, if $w$ and $w^*$ is not parallel, that is  if the angle between $w$ and $w^*$ is not $0$ and $\pi$, then we have 
\begin{equation}\label{eq:strict_inequality}
\frac{d}{dt} 
\|w_{Sob} -w^* \|^2< 
\frac{d}{dt} 
\|w_{L^2} -w^* \|^2 < 0.
\end{equation}
\end{itemize}
\end{theorem}
The proof can be found in Appendix \ref{appendix:convergence_analysis}

\subsection{Sobolev Training with PCGrad}\label{subsec:PCGrad}

%In this section, we briefly introduce PCGrad \cite{yu2020gradient}, of optimization technique designed to address multi-task learning scenarios.

The loss function, \cref{sob_loss}, takes the form of a linear summation of two terms, $\mathcal{L}_{L^2}$ and $\mathcal{L}_{Der}$. It can be seen as a double-task objective, encompassing the prediction of $u$ and its derivatives. Learning two tasks simultaneously may lead to instability during training, especially when there are conflicting gradients—indicating gradients with opposite directions— or a substantial difference in gradient magnitudes \cite{yu2020gradient}.

To address these issues, we adopt PCGrad \cite{yu2020gradient} during training, a technique designed to address multi-task learning scenarios. Here, we briefly explain how to apply PCGrad to our framework.

Let $\bg_{1}:=\partial \mathcal{L}_{L^2} / \partial w$ and $\bg_{2}:=\partial \mathcal{L}_{Der} / \partial w$, respectively. We say $\bg_{1}$ and $\bg_{2}$ \textit{are in conflict} if they are in opposite directions, namely $\bg_{1} \cdot \bg_{2} < 0$. When $\bg_{1}$ and $\bg_{2}$ are not in conflict, the parameters are updated by using the gradient $\bg:=\bg_{1}+\bg_{2}$ as usual. If $\bg_{1}$ and $\bg_{2}$ are in conflict, however, we replace $\bg_{2}$ with $\bg_{2} - (\bg_{1} \cdot \bg_{2} / |\bg_{1}|^{2})\bg_{1}$, essentially subtracting the projection of $\bg_{2}$ onto $\bg_{1}$ from $\bg_{2}$ or vice versa. This subtraction of the projection ensures that $\bg_{1}$ and $\bg_{2}$ are no longer in conflict, resulting in a more stable update of parameters.

In \cref{subsubsec:Ablation_PCGrad}, we conduct an ablation study for the presence and absence of PCGrad in our framework. Specifically, we compare the cases of directly optimizing \cref{sob_loss} and applying PCGrad. It is experimentally confirmed that applying PCGrad performs better than not applying it. More details are in \cref{subsubsec:Ablation_PCGrad}.

\section{Experiments}\label{sec:Experiments}
In this section, we present a variety of experiments to demonstrate the effectiveness of the proposed framework. We compare the performance between the presence and absence of our framework. All the experiments are performed utilizing PyTorch 1.18.0 on Intel(R) Core(TM) i9-10980XE CPU at 3.00GHz and Nvidia GA102GL [RTX A5000].

\paragraph{Baseline models}\label{para:Baseline_Models} We perform experiments on four baselines models: FNO \cite{li2021burigede}, Galerkin Transformer \cite{cao2021choose}, GeoFNO \cite{li2022fourier}, and GNOT \cite{hao2023gnot} with following datasets. We set $K=20$ and $m=2$, respectively. More details in the choice of $K$ and $m$ are described in \cref{subsubsec:Ablation_K,subsubsec:poly_order}. We always set all other implementation details, e.g., each model's number of hidden layers and selections of activation functions, to their default settings.

\paragraph{Datasets}\label{para:Datasets} We adopt six benchmark dataset as following: Darcy2d \cite{li2021burigede}, NS2d \cite{li2021burigede}, NACA \cite{li2022fourier}, Elasticity \cite{li2022fourier}, NS2d-c \cite{hao2023gnot}, and Heat \cite{hao2023gnot}.

Please refer to \cite{li2021burigede,cao2021choose,li2022fourier,hao2023gnot} for more details in baselines and models.

\paragraph{Error Criteria} \label{para:Err_Cri}
 In this section, we present an error analysis, adopting the relative errors in $L_2$-norm as the error criteria, formulated as follows: 
\begin{equation*}
    \frac{1}{{N}_{D}} \sum\limits_{i=1}^{{N}_{D}} \frac{{\|\bu^{NN}_{i} - \bu_{i}\|}_{2}}{{\|\bu_{i}\|}_{2}},
\end{equation*}
where ${N}_{D}$ is the dataset size, $\bu^{NN}_i$ and $\bu_i$ represent the predicted solutions and the target solutions for $i$th data, respectively.

\subsection{Main Results}\label{subsec:MainResults}
%\begin{tabular}[c]{@{}c@{}}The order of \\ Polynomials\end{tabular}
\begin{table*}[t]
\caption{The main results for applying our frameworks on operator learning for a variety of experiments. \textsc{Variables} indicate the physical quantities to predict. The numerical results are the errors before and after applying our framework (\textit{before} $\rightarrow$ \textit{after}). The dashes indicate that the model is not available to handle the dataset.}
\label{table:MainResults}
\vskip 0.15in %%%%%%%%%%%%%%%%%%%
\begin{center}
\begin{small}
\begin{sc}
\begin{tabular}{lccccc}
\toprule
Dataset & Variables & FNO & Geo-FNO$^{\mathbf{\mathrm{A}}}$  & \begin{tabular}[c]{@{}c@{}}Galerkin$^{\mathbf{\mathrm{B}}}$ \\ Transformer\end{tabular} & GNOT$^{\mathbf{\mathrm{B}}}$\\
\midrule
Darcy2d$^{\mathbf{\mathrm{C}}}$    & $u$  &1.09e-2 $\rightarrow$ \textbf{8.25e-3} & {1.09e-2 $\rightarrow$ \textbf{8.25e-3}} & {9.77e-3 $\rightarrow$ \textbf{8.87e-3}} & 1.07e-2 $\rightarrow$ \textbf{9.02e-3} \\
    &  $u$  &1.09e-2 $\rightarrow$ \textbf{8.27e-3} & {1.09e-2 $\rightarrow$ \textbf{8.27e-3}} & {9.44e-3 $\rightarrow$ \textbf{8.98e-3}} & 1.07e-2 $\rightarrow$ \textbf{9.41e-3} \\
    &  $u$  &1.09e-2 $\rightarrow$ \textbf{8.85e-3} & {1.09e-2 $\rightarrow$ \textbf{8.85e-3}} & {1.02e-2 $\rightarrow$ \textbf{9.79e-3}} & 1.08e-2 $\rightarrow$ \textbf{9.85e-3} \\\\
NS2d    & $w$ &1.56e-1 $\rightarrow$ \textbf{1.12e-1} & 1.56e-1 $\rightarrow$ \textbf{1.12e-1} & 1.40e-1 $\rightarrow$ \textbf{7.53e-2} & 
1.38e-1 $\rightarrow$ \textbf{1.01e-1} \\\\
NACA$^{\mathbf{\mathrm{D}}}$    & $v$  &4.21e-2 $\rightarrow$ \textbf{3.73e-2} & 1.38e-2 $\rightarrow$ \textbf{1.17e-2} & 1.94e-2 $\rightarrow$ \textbf{1.67e-2} & 7.57e-3 $\rightarrow$ \textbf{6.00e-3}     \\\\

Elasticity$^{\mathbf{\mathrm{D}}}$   & $\sigma$ &5.08e-2 $\rightarrow$ \textbf{4.51e-2} & 2.34e-2 $\rightarrow$ \textbf{1.57e-2} & 2.31e-2 $\rightarrow$ \textbf{1.71e-2} & 
1.36e-2 $\rightarrow$ \textbf{1.11e-2} \\\\
NS2d-c$^{\mathbf{\mathrm{D}}}$    & $u$ &6.28e-2 $\rightarrow$ \textbf{4.35e-2} & 1.41e-2 $\rightarrow$ \textbf{9.66e-3} & {1.41e-2 $\rightarrow$ \textbf{1.10e-2}} & 1.82e-2   $\rightarrow$ \textbf{1.16e-2} \\
    & $v$ &1.18e-1 $\rightarrow$ \textbf{9.52e-2} & 2.98e-2 $\rightarrow$ \textbf{1.91e-2} &  2.74e-2 $\rightarrow$ \textbf{2.18e-2} & 2.92e-2$\rightarrow$ \textbf{1.70e-2}  \\
    & $p$ &1.14e-2 $\rightarrow$ \textbf{9.74e-3} & 1.62e-2 $\rightarrow$ \textbf{1.07e-2} & 1.88e-2 $\rightarrow$ \textbf{1.38e-2} & 2.66e-2$\rightarrow$ \textbf{1.45e-2}  \\\\
Heat    &$u$ &-& - & - & 4.31e-2 $\rightarrow$ \textbf{3.22e-2} \\

\bottomrule
\end{tabular}
\end{sc}
\end{small}
\end{center}
\vskip -0.1in
\end{table*}

The main results for experiments are described in \cref{table:MainResults}. The results include both before and after applying our framework, described in the form of \textit{before} $\rightarrow$ \textit{after}. For example, the error decreases from 1.09\% to 0.827\% for FNO on the Darcy2d dataset with $141\times141$ grids. The following are several notes for superscripts in \cref{table:MainResults}:
\begin{itemize}
    \item \textbf{\textbf{(A)}} Since Geo-FNO is identically equal to the ordinary FNO on the dataset, which consists of uniform grids, it has the same performance as FNO on Darcy2d and NS2d dataset.
    \item \textbf{\textbf{(B)}} Since the code is not fully reproducible, including the configuration of the dataset, we used our own implementation so that the numerical results might be slightly different from the original one.
    \item  \textbf{\textbf{(C)}} We measure the error on various resolutions of the Darcy2d dataset. From top to bottom, the resolutions are 211, 141, and 85, respectively.
    \item  \textbf{\textbf{(D)}} Since FNO and Galerkin Transformer are only available for datasets with uniform grids, we apply interpolations on NACA, Elasticity, and NS2d-c dataset when we use these models. Specifically, interpolations of NACA and Elasticity dataset are available in \cite{li2022fourier}. NS2d-c dataset is interpolated with $141\times141$ grids.
\end{itemize}
 
Based on the main results, our framework substantially effectively reduces errors across diverse models and datasets. Specifically, the error has decreased by more than 30\% on some tasks. 
This implies that our framework can be applied across a wide range of operator learning tasks, enhancing performance, which is one of the most important advantages of our framework.

\subsection{Further Experiments}\label{subsec:FurtherExperiments}
In this section, we discuss a deeper analysis of several properties of our frameworks with additional experiments as follows. Unless otherwise
specified, we use FNO with the Darcy2d dataset, discretized by $141 \times 141$ grids. In this section, the terminology \textit{Ordinary} implies the framework of the basic training framework, optimizing \cref{L2_loss} only.

\subsubsection{Robustness against Noise}\label{subsubsec:Robust_Noise}
We first investigate the robustness of our approach in the presence of varying noise intensities. We inject i.i.d. Gaussian noise $\mathcal{N}(0, \sigma^2I)$ noise to each target data while training, where $\sigma$ denotes the noise intensity. We set $\sigma$ to be adjusted proportionately to the target data range. Specifically, we vary $\sigma$ from 1.5\% to 3.0\% of the difference between the maximum and minimum of $u$. We repeat five times for generating noise per each $\sigma$. We compare three frameworks: ours, the ordinary method, and the ordinary method incorporating Finite Difference Methods (FDM). The results are described in \cref{table:NoiseDarcy2d}. As outlined in \cref{table:NoiseDarcy2d}, our framework demonstrates superior performance in all scenarios, indicating its remarkable robustness against noise. Specifically, the ordinary method has been identified as highly susceptible to noise. In particular, as $\sigma$ increases, the error more than doubles ($1.09\% \rightarrow 2.26\%$). We find that applying Sobolev Training results in much more stable learning against the presence of noise. We also note that utilizing our framework in Sobolev Training leads to much better performance than simply applying FDM.

\begin{table}[t]
\caption{Quantitative results on the Darcy2d dataset with noise are presented. The experiments are repeated five times for each value of $\sigma$. The average errors are described.}
\label{table:NoiseDarcy2d}
\vskip 0.15in
\begin{center}
\begin{small}
\begin{sc}
\begin{tabular}{cccc}
\toprule
\begin{tabular}[c]{@{}c@{}}Noise \\ Intensity\end{tabular} & Ordinary & Ours & \begin{tabular}[c]{@{}c@{}}Ordinary \\ + FDM\end{tabular}\\
\midrule
0.00\%  & 1.09e-02  & \textbf{8.27e-03} & 9.67e-03 \\
1.50\%  & 1.29e-02  & \textbf{9.04e-03} & 9.76e-03 \\
2.00\%  & 1.60e-02  & \textbf{9.28e-03} & 1.01e-02 \\
2.50\%  & 1.94e-02  & \textbf{9.30e-03} & 1.03e-02 \\
3.00\%  & 2.26e-02  & \textbf{9.62e-03} & 1.06e-02 \\
\bottomrule
\end{tabular}
\end{sc}
\end{small}
\end{center}
\vskip -0.1in
\end{table}

\subsubsection{Hyperparameter Selection for $K$}\label{subsubsec:Ablation_K}

\begin{figure}[t]
\vskip 0.2in
\begin{center}
\centerline{\includegraphics[width=0.85\columnwidth]{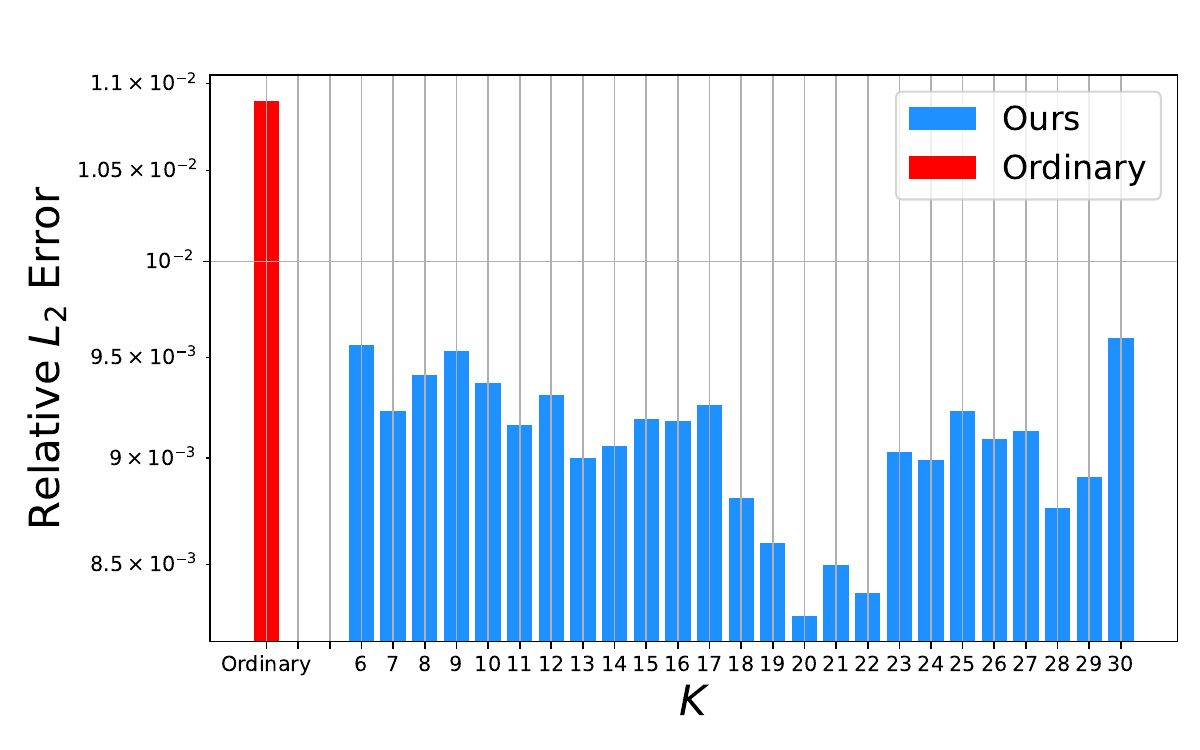}}
\caption{The relative error for each $K$ from 6 to 30. The y-axis is the relative $L_2$ error; the x-axis is the value of $K$.
}
\label{fig:RelativeError_K}
\end{center}
\vskip -0.2in
\end{figure}

% Secondly, we conduct an ablation study for $K$, one of the most important hyperparameters in our framework. We investigate how the relative error varies depending on the value of $K$. We vary the value of variable $K$, recording the corresponding errors in \cref{fig:RelativeError_K}. As described in \cref{fig:RelativeError_K}, our framework has the best performance when $K=20$. Around K=20, it is experimently confirmed that the performance worsens if K increases or decreases.

% This result can be interpreted by a perspective of the receptive field of our framework. Before addressing that point, let us first remind that derivative is a kind of \textit{local} property. The derivative of a function at a query point only depends on its local neighborhoods, and $K$ determines the range of such neighborhoods. Thus, the choice of $K$ can be regarded as adjusting the range of the receptive field for calculating the derivatives, so it is quite natural that a moderate, neither too large nor too small $K$ is most suitable.

$K$ is one of the most important hyperparameters in our framework. We investigate how the relative error varies depending on the value of $K$. We vary the value of $K$, recording the corresponding errors in \cref{fig:RelativeError_K}. As illustrated in \cref{fig:RelativeError_K}, our framework performs best when $K=20$. Around $K=20$, it is experimentally confirmed that the performance worsens if $K$ increases or decreases.

This result can be interpreted from the perspective of the receptive field. Before addressing that point, we note that the derivative is a kind of \textit{local} property. The derivative of a function at a query point depends only on its local neighborhoods, and $K$ determines the range of such neighborhoods. Thus, the choice of $K$ can be regarded as adjusting the range of the receptive field for calculating the derivatives. Consequently, it is natural that a moderate, neither too large nor too small, value for $K$ is most suitable.

%\begin{tabular}[c]{@{}c@{}}The order of \\ Polynomials\end{tabular}
\subsubsection{Hyperparameter Selection for the order of polynomial}\label{subsubsec:poly_order}
\begin{table}[t]
\caption{The relative error for Darcy2d-dataset corresponding to the order of the polynomial $p(x)$, ranging from 1 to 4.}
\label{table:AblationStudy_m}
\vskip 0.15in
\begin{center}
\begin{small}
\begin{sc}
\begin{tabular}{ccccc}
\toprule
$m$ & 1 & 2  & 3 & 4\\
\midrule
Errors    & 9.21e-03 & \textbf{8.25e-03} & 9.33e-03 & 3.30e-02 \\
\bottomrule
\end{tabular}
\end{sc}
\end{small}
\end{center}
\vskip -0.1in
\end{table}

We next conduct an ablation study for $m$, the order of $p(x)$ defined by \cref{eq:polynomial}, to choose the most suitable value. We systematically vary the value of $m$ from 1 to 4 for Sobolev Training, recording the corresponding errors in \cref{table:AblationStudy_m}. As described in \cref{table:AblationStudy_m}, we observe that the case of $m=2$ has the best performance. We also observe a sharp decline in performance when $m=4$, as an order that is too high, increases learning complexity, resulting in training difficulties.

\subsubsection{Ablation Studies for the regularity conditions}\label{subsubsec:Regularity_Conditions}
One of the most important conditions in our framework is the regularity condition, specifically $u \in W^{M, 2}(\Omega)$. We investigate the influence of the regularity condition on performance by intentionally omitting it.

\begin{figure}[t]
\vskip 0.2in
\begin{center}
\centerline{\includegraphics[width=1.0\columnwidth]{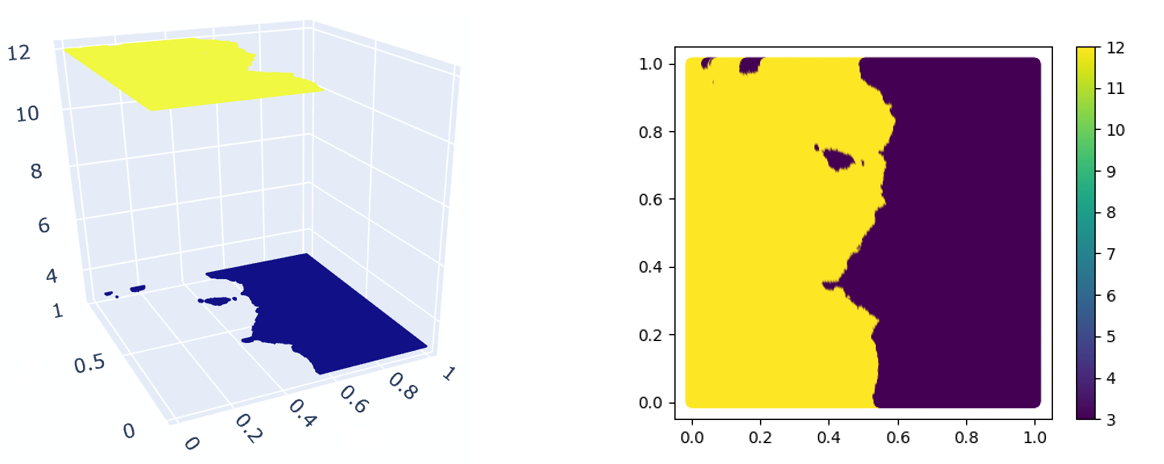}}
\caption{Visualization of coefficient function $a$ of the Darcy2d-dataset. Due to jump discontinuity, $a \in L^{\infty}((0,1)^{2};\mathbb{R}_{+})$ fails to belong to $W^{M, 2}((0,1)^{2})$ for all $M \in \mathbb{N}$.
}
\label{fig:a_discontinuous}
\end{center}
\vskip -0.2in
\end{figure}

\begin{table}[t]
\caption{Relative errors for the inverse problem of Darcy2d. We compare three training types: ours and the ordinary training (with/without) Finite Difference Methods.}
\label{table:Inverse_Darcy2d}
\vskip 0.15in
\begin{center}
%\begin{small}
\begin{sc}
\begin{tabular}{cc}
\toprule
\begin{tabular}[c]{@{}c@{}}Training \\ Type\end{tabular} & Error\\
\midrule
Ours & 4.78e-02\\
Ordinary    & 2.47e-02\\
Ordinary + FDM    & 6.93e-02\\
\bottomrule
\end{tabular}
\end{sc}
%\end{small}
\end{center}
\vskip -0.1in
\end{table}

Here, we address the inverse problem of Darcy2d, aiming to learn a mapping from the solution $u$ to the coefficient $a$. Unlike the forward problem, the target $a \in L^{\infty}((0,1)^{2};\mathbb{R}_{+})$ does not belongs to $W^{M, 2}((0,1)^{2})$ for any $M \in \mathbb{N}$ (see \cref{fig:a_discontinuous}).

We conduct the experiments for the Galerkin Transformer, which performs best on the inverse problem of Darcy2d among baselines. The comparisons involve our approach and the ordinary training, both with and without FDM. As described in \cref{table:Inverse_Darcy2d}, our ordinary training with FDM performs worse than the ordinary one. This is interpreted as a result of the absence of regularity conditions for the target function $a$. Since $a$ is not differentiable due to jump discontinuity, applying Sobolev Training in this case is unsuccessful. This experiment confirms that the regularity condition is necessary to apply Sobolev Training successfully.

\subsubsection{Ablation Studies for PCGrad}\label{subsubsec:Ablation_PCGrad}

\begin{table}[t]
\caption{Comparison of relative errors on Darcy2d dataset when applying and not applying PCGrad. \\ \textbf{\textsc{With PCGrad}} : Sobolev Training with applying PCGrad \\ \textbf{\textsc{Without PCGrad}} : Sobolev Training optimizing \cref{sob_loss} directly rather than applying PCGrad}
\label{table:Ablation_PCGrad}
\vskip 0.15in
\begin{center}
%\begin{small}
\begin{sc}
\begin{tabular}{cc}
\toprule
\begin{tabular}[c]{@{}c@{}}Training \\ Type\end{tabular} & Error\\
\midrule
with PCGrad & \textbf{7.34e-03}\\
without PCGrad  & 8.01e-03\\
Ordinary    & 8.40e-03\\
\bottomrule
\end{tabular}
\end{sc}
%\end{small}
\end{center}
\vskip -0.1in
\end{table}

Our last experiment is to evaluate the impact of applying PCGrad to Sobolev Training on performance. As aforementioned in \cref{subsec:PCGrad}, we conduct an ablation study for PCGrad to confirm the effectiveness of PCGrad experimentally.
We estimate the relative error on the Darcy2d dataset using the Galerkin Transformer, comparing the cases of applying and not applying PCGrad, i.e., directly optimizing \cref{sob_loss}. The results are described in \cref{table:Ablation_PCGrad}. According to \cref{table:Ablation_PCGrad}, the error consistently decreases compared to the ordinary case when Sobolev Training, and particularly, it is confirmed that the performance is further enhanced when combined with PCGrad compared to when not using it.

\section{Conclusion}\label{sec:Conclusion}
In this study, we explore the impact of Sobolev Training on operator learning. Theoretically, we demonstrate that, in the case of a one-layer ReLU network with several assumptions on the kernel function, incorporating a derivative term accelerates the convergence rate. Our numerical experiments further reveal that Sobolev Training significantly influences a variety of datasets and models. Our future work aims to extend these theoretical findings to more general cases and refine the proposed derivative approximation algorithm. This direction promises to enhance the robustness and applicability of our approach in the field of operator learning.

\newpage
\bibliography{reference.bib}
\bibliographystyle{icml2024}

%%%%%%%%%%%%%%%%%%%%%%%%%%%%%%%%%%%%%%%%%%%%%%%%%%%%%%%%%%%%%%%%%%%%%%%%%%%%%%%
%%%%%%%%%%%%%%%%%%%%%%%%%%%%%%%%%%%%%%%%%%%%%%%%%%%%%%%%%%%%%%%%%%%%%%%%%%%%%%%
% APPENDIX
%%%%%%%%%%%%%%%%%%%%%%%%%%%%%%%%%%%%%%%%%%%%%%%%%%%%%%%%%%%%%%%%%%%%%%%%%%%%%%%
%%%%%%%%%%%%%%%%%%%%%%%%%%%%%%%%%%%%%%%%%%%%%%%%%%%%%%%%%%%%%%%%%%%%%%%%%%%%%%%
\newpage

\appendix
\onecolumn
\section{Notations}
Let us revisit some of the notations mentioned in Section \ref{subsec:notation} and introduce additional notations to explain details on the algorithm and the theorem proposed in Section \ref{sec:proposed_method}.
\begin{itemize}
\item For $m_1, m_2 \in \mathbb{Z}$, we use the notation $[m_1, \cdots, m_2]$ to denote an array of integers from $m_1$ to $m_2$.
    \item $\alpha = (\alpha^1, \cdots, \alpha^{n}) \in\left(\mathbb{N}\cup\{0\}\right)^n$ denotes a multi-index.
    
    \item We use the notation $|\alpha| = \sum_{i=1}^{n} \alpha^i$ and $\alpha! = (\alpha^1)!(\alpha^2)!\cdots(\alpha^n)!$.
    
    \item $D_x^{\alpha}$ represents a derivative in the $x$ direction of order $\alpha$.
    \item We denote $D_x u = (\partial_1 u, \cdots, \partial_n u)$, where each $\partial_i u$ denotes the partial derivative of $u$ with respect to the $i^{th}$ variable.
    \item We use $w\in \mathbb{R}^n$ to denote the weight of the neural network, and $D_w u(x;w)$ represents the derivative of $u$ with respect to the weight $w$.
    \item For $x\in \mathbb{R}^n$, we denote $x=(x^1, x^2, \cdots, x^n)$ and $\|x\| =\sqrt{(x^1)^2+\cdots +(x^n)^2} $.
    \item We use the notation $f \stackrel{\gamma}{\approx} g$, if there     a universal constant $C=C(\gamma)>1$ such that
    $$
     \frac{1}{C} g\leq f\leq C g.
    $$
    \item For a bounded domain $\Omega\subset \mathbb{R}^n$, we denote 
    $$
    |\Omega| = \int_{\Omega}dx.
    $$
\end{itemize}

\section{Details of Algorithm \ref{alg:gradient_est}}
This section offers a more detailed explanation of the derivative approximation algorithm. For instance, when $n=2$ and $m=2$, we consider the two-dimensional case where the degree of the approximating polynomial is less than or equal to $2$.
We denote $x=(x^{1}, x^{2})$ and mesh points are denoted as $x_j = (x_j^1, x_j^2)$.
The basis vector of $\Pi_{2}^{2}$  is 
\begin{equation*}
%\label{eq:basis_mtrix}
b_{j}(x) =  
\begin{pmatrix}
        1\\
        x^{1}-x_{j}^{1}\\
        x^{2}-x_{j}^{2} \\
        (x^{1}-x_{j}^{1})^2\\
        (x^{1}-x_{j}^{1})(  x^{2}-x_{j}^{2}) \\
        (  x^{2}-x_{j}^{2}) ^2
        \end{pmatrix}.
\end{equation*}
For each $x_j\in \{x_{j}\}_{j=1}^{J}$,  we find  K nearest points $\{x_{KNN, j,1}, \cdots, x_{KNN, j, K}\}\subset  \{x_{j}\}_{j=1}^{J}$ and calculate the coefficient vector $c_j$ using \cref{eq:coefficient} with 
\begin{equation}\label{weight_def}
\omega(t) = \left(1-\frac{t}{D}\right)^{4}\left(\frac{4t}{D}+1\right)
\quad \text{and } \quad D = 1.1 * \max_{j, k}\|x_j - x_{KNN, j,k}\|.
\end{equation}
Since the average time complexity of widely used KNN algorithms (KD-tree, Ball-Tree) is  $O(\log N )$ and matrix inversion has a time complexity of $O(K^{3})$ for $K\times K$ matrix, Algorithm \ref{alg:gradient_est} has an average time complexity of $O(K^{4} N\log N )$. Since $K$ is very small compared to $N$ and fixed number, the time complexity of Algorithm \ref{alg:gradient_est} is $O(N\log N)$ considering $K$ as a constant. 
\section{Proof of theorems  in Section \ref{sec:proposed_method} }\label{appendix:proposed_method_proof} 
This section provides detailed proof of theorems in Section \ref{sec:proposed_method}.
% We assume that $x_i$
\subsection{Proof of Lemma \ref{lem:grad_convergence}}\label{appendix:prove_derivative_lem}
\begin{proof}
Let us assume $u \in C^{M}(\Omega)$ for some $M \in \mathbb{N}$.
We define the matrices $W_j\in \mathbb{R}^{K \times K} $ and $E_j \in  \mathbb{R}^{K \times I} $ as follows: 
\begin{align*}
W_j &= \text{diag}(\omega_{j,1}, \ldots, \omega_{j,K}) \quad \text{where}\quad \omega_{j,k} = \omega(|x_{j} - x_{j,\text{KNN},k}|),\\
E_{kj} &= b_j(x_{j,\text{KNN},k}) \quad \text{for}\quad  k \in [1, \ldots, K], \  j \in [1, \ldots, J].
\end{align*}
Further, we define $A_j:= E_j^T W_j E_j$ and $A_{j, \tilde{j} \xleftarrow{} w}$ to be the matrix formed by replacing the $\tilde{j}^{th}$ column of $A_j$ with the vector $w$.

Now, for each $j \in [1, J]$ and a multi-index $\alpha_{\tilde{j}}$ with $|\alpha_{\tilde{j}}|= m$, we apply Theorem 3.1 of \cite{lipman2006error} to obtain
\begin{equation}\label{eq:pointwise_approx}
c_{j,\alpha_{\tilde{j}} }  - D_x^{ \alpha_{ \tilde{j}   }  } u(x_j) 
= \sum_{|\alpha| = M } \sum_{k=1}^{K} \frac{1}{ 
   \alpha! }
D^{\alpha}_{x}u(x_j+\theta_{j,k}(x_j-x_{KNN, j, k}))
  |x_j-x_{KNN, j,k}|^{\alpha} 
 \frac{\det(A_{j, \tilde{j} \xleftarrow{ } w_kb_j(x_k) })}{\det(A_j)}
\end{equation}
for some $0 \leq \theta_{KNN, j, k} \leq 1$.
We  denote $s=\sum_{\tilde{j}=1}^{I}|\alpha_{\tilde{j}}|$
and use Remark 4 of \cite{liang2013solving} to have
$$
\det(A_{j, \tilde{j} \xleftarrow{ } w_kb_i(x_k) })\approx O(h^{2s - |\alpha_{\tilde{j}}|})  \quad \text{and}
\quad  
\det(A ) \approx  O(h^{2s}).
$$
From the relations above, it is directly checked that 
$$
\frac{\det(A_{j, \tilde{j} \xleftarrow{ } w_kb_j(x_k) })}{\det(A_j)}\leq Ch^{-|\alpha_{\tilde{j}}|}.
$$
First, take a square on the both side of \cref{eq:pointwise_approx}, then take average over $j=1\cdots J$ and take summation for all multi-index satisfying $|\alpha_{\tilde{j}}|=m$, finally  then use assumption \cref{eq:approx_} to conclude
\begin{align*}
\frac{1}{J}
\sum_{|\alpha_{\tilde{j}}|=m}\sum_{j=1}^{J}|c_{j, \alpha_{\tilde{j}} } - D^{\alpha_{\tilde{j}}}_x u(x_j)|^2 
&\leq Ch^{M-m}\|u \|_{W^{M,2}(\Omega)}.
\end{align*}
Then the proof is complete by approximating $u\in W^{M,2}(\Omega)$ by $C^M(\Omega)$ function using Theorem 5.3.2 of \cite{evans2022partial}.
\end{proof}

\subsection{Proof of Theorem  \ref{lem:convergence_analysis} }\label{appendix:convergence_analysis}
We need the following auxiliary lemmas to prove the main theorem of this paper.
One can achieve the result from a direct computation.
\begin{lemma}\label{lem:derivative_rule}
    Suppose that 
    $$
    A \in \mathbb{R}^{1\times d}, 
    \quad 
    B\in \mathbb{R}^{J \times d }\quad \text{and}\quad   w\in \mathbb{R}^{d\times 1 }.
    $$
    If $f(w):= AwBw$, then we establish
    $$
    D_w f(w) = Aw B +BwA.
    $$
\end{lemma}

The proof of the following lemma can be found in Section 15.10 in \cite{ord1994kendall}.
\begin{lemma}\label{lem:bivariate_normal}
    Suppose that 
    $$
x, y \sim \mathcal{N}(0,1)
    $$
    with a correlation $\rho  \in [0,1]$. Then we denonstrate
    $$
    \mathbb{P}\left[x>0 , y>0 \right] =\frac{1}{4} +\frac{1}{2\pi} \arcsin{\rho}
    $$
\end{lemma}
Let us denote
$\theta(w_1, w_2) \in [0, \pi]$ by the angle between $w_1$ and $w_2\in \mathbb{R}^n$ .
 For the notational convenience, we denote 
\begin{align*}
 C (e,w) &=  \frac{1}{2\pi }
 [(\pi -  \theta(e, w))w+\|w\| \sin \theta(e,w) e] \in \mathbb{R}^{d\times  1},\\
\alpha(w_1, w_2) &= 
w_1^TC
\left(\frac{w_1}{\|w_1\|} , w_2\right) 
= \frac{\|w_1\| \|w_2\|}{2\pi}
[(\pi -  \theta(w_1, w_2)) \cos \theta(w_1,w_2)+  \sin \theta(w_1,w_2) ] 
\in \mathbb{R}.
\end{align*}
Furthermore, we denote
\begin{align*}
\phi_0(w_1,w_2)  & = \frac{1}{2\pi} [(\pi -  \theta(w_1, w_2)) \cos \theta(w_1,w_2)+  \sin \theta(w_1,w_2) ] \\
\phi_1(e,w)  & = \frac{1}{2\pi }(\pi -  \theta(e, w))  \\
\phi_2(e,w)  & = \frac{1}{2\pi }\sin \theta(e,w)    .
\end{align*}
Note that one can directly check that
\begin{align*}
    C(e,w) = \phi_1(e,w) w + \phi_2(e,w) \|w\|e \quad \text{and}\quad \alpha(w_1,w_2) = \|w_1\|\|w_2\| \phi_0(w_1,w_2).
\end{align*}
We need the following results from 
Theorem 1 of \cite{tian2017analytical}.
\begin{lemma}\label{lem:XTX_formula}
    Let us denote
    \begin{equation}\label{eq:FXew}
    F(X, e,w) := X^T D(X, e)D(X, w)Xw.
    \end{equation}
    Then we find
    $$
    \mathbb{E}_X\left[F(X, e, w)\right] = \frac{J}{2\pi } 
    [(\pi -  \theta(e, w))w+\|w\| \sin \theta(e,w) e] = J C(e,w).
    $$
\end{lemma}

\begin{lemma}\label{lem:g_less_0}
    Let us define
    $$
    g(\theta) :=(\cos \theta -1)((\pi-\theta)(\cos\theta -1)+2\sin \theta).
    $$
    Then $g(\theta)\leq 0$ for all $\theta \in [0, \pi]$.
\end{lemma}
\begin{proof}
    Cleary, we observe that $g(0)=g(\pi) = 0$. Also, from a direct computation, we find
$$
g'(\theta) = 
\sin \theta (-3 \sin\theta +2 (\pi  - \theta)(1-\cos \theta )).
$$
If $g'(\theta_0)= 0 $ and $0< \theta_0<\pi$, then $\sin \theta_0 \not =0$ and we have
\begin{equation}\label{eq:theta_0_eq}
3\sin\theta_0 =2(\pi-\theta_0)(1-\cos\theta_0).
\end{equation}
 By \cref{eq:theta_0_eq} and some mathematical manipulation, 
 we find  
    $$
    g(\theta_0)=\frac{\pi-\theta_0}{3} ( \cos \theta_0 -1) ( 1 - \cos \theta_0)<0.
    $$
    The strict inequality holds because $\cos \theta_0\not = 0$.
Taking twice derivative of $g(\theta)$, we find
$$
g''(\theta) 
=
-2((\theta
- \pi)  (2\cos^2\theta  -\cos\theta - 1) + \sin \theta(1+2\cos\theta)).
$$
By \cref{eq:theta_0_eq} again,  we find
\begin{align*}
g''(\theta_0) 
&= -2\left
((\pi-\theta_0)(2\cos^2\theta_0 -\cos\theta_0 -1) +\frac{2}{3}(\pi-\theta_0)(1-\cos \theta_0)(2\cos\theta_0+1)
\right)\\
&= -\frac{2}{3}
(\pi-\theta_0)(2\cos^2\theta_0 -\cos\theta_0 -1)\\
&=
-\frac{2}{3}
(\pi-\theta_0)(2\cos\theta_0 +1)(\cos\theta_0 -1).
\end{align*}
Note that if $0<\theta_0 <\frac{5\pi}{6}$, then $-\frac{1}{2} \leq \cos\theta_0\leq 1$, which implies that $g''(\theta_0)> 0$, as $(2t+1)(t-1)< 0$ for $-\frac{1}{2}<t<1$.
On the other hand, if $\frac{5\pi}{6}\leq \theta_0  < \pi$, then we have
\begin{align*}
    g'(\theta_0) = \sin \theta_0  (\pi-\theta_0)(1-\cos \theta)\left( 2-\frac{3\sin \theta_0 }{(\pi-\theta_0)  (1-\cos \theta_0)}\right).
\end{align*}
Reminding $\sin \theta > \theta$ for all $\theta>0$, it is evident that
$$
\frac{\sin \theta_0 }{\pi-\theta_0} =\frac{\sin (\pi-\theta_0) }{\pi-\theta_0}  <1
\quad\text{and}
\quad
\frac{3}{2}\leq 1-\cos\theta_0\leq 2, 
$$
 to get $g'(\theta_0)>0$. 
 These calculations imply that $\theta_0$ is not the extreme point. Therefore, all the extreme points are local minimums, implying that there exists only one local minimum of $g(\theta)$ at $\theta_0\in (0,\pi)$ and $g(\theta_0)<0$. We can finally conclude that $g(\theta)\leq 0$ for all $0\leq \theta\leq \pi$.
\end{proof}
The following lemma can be directly from elementary calculation.
\begin{lemma}\label{lem:sign_of_3_oder_poly}
Let us define a real-valued 3rd-order polynomial 
$$
f(t): = A t^3 - Bt^2 -Ct+D.
$$
If $A>0$ and $B^2 +3AC >0$, then $f(t)$ has a  local minimum at 
$$
t_0 = \frac{B+\sqrt{B^2  + 3AC}}{3A},
$$
and the value of the  local minimum value is
$$
f(t_0) = \frac{1}{27A^2}\left(
27A^2D -2B^3 - 9ABC - 2(B^2+3AC)\sqrt{B^2+3AC}
\right).
$$
 \end{lemma}

% The following lemma estimates the maximum value of a specific auxiliary function. 
% \begin{lemma}\label{lem:bounded_R}
% Suppose that $f$ is defined on $t\in [0,\pi]$ satisfying
% $$
% f(t) = \cos t + \frac{\sin t}{\pi - t} \quad \text{for }0\leq x<\pi \quad \text{and}\quad f(\pi) = 0.
% $$
% Then we have 
% $$
% \max_{t\in [0,\pi]} f(t)\leq \frac{10\pi^2}{81}.
% $$
% \end{lemma}
% \begin{proof}
%     Taking derivative, we have
%     $$
%     f'(t) =\frac{(1-(\pi-t)^2)\cos t +(\pi-t)\sin t    }{(\pi-t)^2}.
%     $$
% We denote $g(t) :=(1-(\pi-t)^2)\sin t +(\pi-t)\cos t$. Then 
% $g(0)= \pi>0$
% and 
% $$
% g\left(\frac{\pi}{9}\right) < 0< g\left(\frac{5\pi}{36}\right) .
% $$
% Therefore $f'(t_0)=0$  for some $\frac{\pi}{9}<t_0<\frac{5\pi}{36}$.
% Using this observation and the fact that $g(t_0)=0$, one can conclude
% $$
% f(t_0)= \cos t_0 +\frac{\sin t_0}{\pi-t_0} = (\pi-t_0)\sin t_0 
% \leq \frac{8\pi}{9}\frac{5\pi}{36}=\frac{10\pi^2}{81}.
% $$
% Here, we used the fact that $\sin t \leq t$ for $t\geq 0$.
% \end{proof}
% We will utilize the following lemma without providing its proof, as the analytical proof is excessively technical and falls outside the scope of this work.
% \begin{lemma}\label{lem:decreasing}
%     Let us define
%     $$
%     f(t) = \frac{(\pi -t )((\pi-t)+ \cos t \sin t )}{(\pi-t)\cos t +\sin t }.
%     $$
%     Then, $f(t)$ is a decreasing function. 
% \end{lemma}
% See Appendix \ref{appendix:visual_loss} for more details about the properties of the function described in the lemma.

\textbf{Proof of Theorem \ref{lem:convergence_analysis}}\label{appendix:lem_}
\begin{proof}
  In this work, we consider a real-valued function; that is, we  assume that
$$
x\in \mathbb{R}^d \quad \text{and}\quad  u(x), v(x)\in \mathbb{R}.
$$
    Consider the sets of input and output function pairs, denoted by $\{u_k(x), v_k(x)\}_{k=1}^{N}$. 
    Additionally, let $\{x_j\}_{j=1}^{J}$ and $\{y_l\}_{l=1}^{\tilde{J}}$ represent query points that are sampled from a spherical  Gaussian distribution $\mathcal{N}(0, I)$.
    We also assume that $x_j$ and $y_l$ are independently sampled for all $j$ and $l$.
    Let $
X = [x_1,x_2, \cdots ,x_J]\in \mathbb{R}^{J\times n}$ and $Y=[y_1,y_2, \cdots ,y_{\tilde{J}}]\in \mathbb{R}^{\tilde{J}\times n}$  be a matrix representation of each query points. 
We represent each $u_k$ using a Monte Carlo approximation, as in 
   \begin{equation}\label{eq:u_k}
   u_k(x_j) 
   =
   \frac{1}{|\Omega|}\int_{\Omega} \kappa(x_j,y) v_k(y)d\nu (y)
   \approx 
 \frac{1}{\tilde{J}} \sum_{l=1}^{ \tilde{J}} 
\mathbb{E}_Y\left[\kappa(x_j, y_l)v_k(y_l) \right].
   \end{equation}
   As explained in Section \ref{sec:proposed_method}, we assume $v_k(y_l)$ and $\kappa(x,y)$ can be expressed as  is given as 
   \begin{equation}\label{eq:v_kappa_assume}
   v_k(x) = \mu_k \sigma((w^*)^Tx )
   \quad \text{and}
   \quad
   \kappa(x,y)= \sigma(w^Tx)\sigma(w^Ty).
   \end{equation}
% We denote 
% \begin{align*}
% V_{L^2}(t) := &\frac{d}{dt} 
% (\|w_{Sob}(t) -w^* \|^2)\leq 
% \frac{d}{dt} 
% (\|w_{L^2}(t) -w^* \|^2)\leq 0.
% \end{align*}
   Using the matrix representation \cref{eq:u_k}, assumption \cref{eq:v_kappa_assume}  and  Lemma \ref{lem:XTX_formula}, we have 
\begin{equation}\label{eq:u_kXW}
    \begin{aligned}
        u_{k}(X; w) 
        &=  \frac{\mu_{k}  (w^*)^T }{\tilde{J} } \mathbb{E}_{Y}\left[
         Y^TD(Y, w^*)D(Y,w)Yw
         \right]
         D(X,w)Xw\\
        & = \frac{\mu_{k}  (w^*)^T }{\tilde{J} } \mathbb{E}_{Y}\left[
         F\left(Y, \frac{w^*}{\|w^*\|} , w\right)
         \right]
         D(X,w)Xw\\
         &=\frac{\mu_k(w^*)^T}{2\pi}
         \left[
         (\pi-\theta(w, w^*) ) w +\|w\| \sin \theta(w, w^*)  \frac{w^*}{\|w^*\|}
         \right] D(X,w)Xw\\
        &=\mu_k(w^*)^T C\left(\frac{w^*}{\|w^*\|}, w\right) D(X,w)Xw\\
         &=\mu_k\alpha(w, w^*) D(X,w)Xw.
    \end{aligned}
\end{equation}
Applying Lemma \ref{lem:derivative_rule} to $u_k(X;w)$, we find 
\begin{equation}\label{eq:Du_kXW}
    \begin{aligned}
       D_w u_k(X;w) 
       & =\mu_k\alpha(w, w^*) D(X,w)X +  \frac{\mu_k }{\tilde{J}} D(X,w)X w \mathbb{E}_{Y}\left[ (w^*)^T  Y^T D(Y,w^*)D(Y,w)Y\right] .
   \end{aligned}
\end{equation}
For $w_1, w_2 \in \mathbb{R}^d$, one use \cref{eq:u_kXW,eq:Du_kXW,eq:FXew} and Lemma \ref{lem:XTX_formula}  to have  
\begin{align*}
&\left[
 D_w u_k(X;w_2 )\right]^ T u_k(X;w_1 )\\
 & =| \mu_{k}|^2 \alpha(w^*, w_1)\alpha(w^*, w_2)  X^T D(X,w_2) D(X,w_1)Xw_1 \\
&\quad +   \frac{\mu_k^2 \alpha(w^*, w_1) }{\tilde{J}}
\mathbb{E}_{Y}\left[ Y^T D(Y,w_2)D(Y,w^*)Y w^* \right]w_2^TX^TD(W,w_2)D(X,w_1)Xw_1\\
&=| \mu_{k}|^2 \alpha(w^*, w_1)\alpha(w^*, w_2)   F\left(X, \frac{w_2}{\|w_2\|},w_1\right) + \frac{| \mu_{k}|^2  \alpha(w^*, w_1)}{\tilde{J}} E_Y[ F\left(Y,\frac{w_2}{\|w_2\|}, w^* \right)] 
w_2^TF\left(X,\frac{w_2}{\|w_2\|}, w_1\right)\\
& = 
| \mu_{k}|^2 \alpha(w^*, w_1)\alpha(w^*, w_2)   F\left(X, \frac{w_2}{\|w_2\|},w_1\right) 
+
| \mu_{k}|^2  \alpha(w^*, w_1)
C\left(\frac{w_2}{\|w_2\|}, w^* \right) 
w_2^TF\left(X,\frac{w_2}{\|w_2\|}, w_1\right)\\
& = 
| \mu_{k}|^2
\left( 
\alpha(w_2, w^*)
I_{d}
+
C\left(\frac{w_2}{\|w_2\|}, w^* \right) 
w_2^T \right)\alpha(w^*, w_1) F\left(X,\frac{w_2}{\|w_2\|}, w_1\right)
.
\end{align*}
This leads to 
\begin{align*}
&\left[
 D_w u_k(X;w)\right]^ T 
 \left( u_k(X;w)
 - 
 u_k(X;w^*) \right)\\
&=| \mu_{k}|^2
\left( 
\alpha(w, w^*) 
I_{d}
+
C\left(\frac{w}{\|w\|}, w^* \right) 
w^T \right)
\left(
\alpha(w, w^*) F\left(X,\frac{w}{\|w\|}, w\right) 
-
\alpha(w^*, w^*) F\left(X,\frac{w}{\|w\|}, w^*\right) 
\right).
\end{align*}
Therefore, it follows that
\begin{align*}
    \mathbb{E}_{X} \left[\frac{\partial \mathcal{L}_{L^2} }{\partial w} \right]
    = \frac{1}{NJ}\sum_{k=1}^{N} \left[
 D_w u_k(X;w)\right]^ T 
 \left( u_k(X;w)
 - 
 u_k(X;w^*) \right),
 \end{align*}
 where $\mathcal{L}_{L^2}$ is the loss function defined as in \cref{L2_loss}.
 If $w$ is updated by $\cref{eq:gradient_flow}_1$, 
 it is directly checked that
 \begin{align*}
    & \frac{1}{2}\frac{d}{dt}\|w - w^*\|^2\\ 
     &=  (w - w^*)^ T  \frac{dw}{dt}\\
     & =  - (w - w^*)^ T  \mathbb{E}_{X} \left[\frac{\partial \mathcal{L}_{L^2} }{\partial w} \right]\\
     &= - \frac{1}{NJ}\sum_{k=1}^{N}| \mu_{k}|^2 (w - w^*)^T \left( 
\alpha(w, w^*) 
I_{d}
+
C\left(\frac{w}{\|w\|}, w^* \right) 
w^T \right)\\
     &\quad \times   
\left(
\alpha(w, w^*) \mathbb{E}_X\left[F\left(X,\frac{w}{\|w\|}, w\right) \right]
-
\alpha(w^*, w^*) 
\mathbb{E}_X\left[F\left(X,\frac{w}{\|w\|}, w^* \right) \right]
\right)\\
&= - \frac{1}{N}\sum_{k=1}^{N}| \mu_{k}|^2 (w - w^*)^T \left( 
\alpha(w, w^*)  
I_{d}
+
C\left(\frac{w}{\|w\|}, w^* \right) 
w^T \right)\\
     &\quad \times   
\left(
\alpha(w, w^*) C\left(\frac{w}{\|w\|}, w \right)
-
\alpha(w^*, w^*) 
C\left(\frac{w}{\|w\|}, w ^*\right)
\right).
 \end{align*}
For each $k=1\cdots N$, it follows that 
\begin{align*}
&(w-w^*)^T \left( 
\alpha(w, w^*) I_{d}
+
C\left(\frac{w}{\|w\|}, w^*  \right) 
w^T \right)\\
&= (w-w^*)^T 
\left(
\|w\|\|w^*\|\phi_0(w,w^*) + 
\phi_1(w,w^*)w^*w^T + \phi_2(w,w^*)\frac{\|w^*\|}{\|w\|}ww^T
\right)\\
&= \|w\|\|w^*\|\phi_0(w,w^*) (w-w^*)^T 
+ \phi_1(w,w^*)w^T w^*w^T +\phi_2(w,w^*)\|w^*\|\|w\|w^T  \\
&\quad 
-\phi_1(w,w^*)(w^*)^T w^*w^T 
-\phi_2(w,w^*)\frac{\|w^*\|}{\|w\|}(w^*)^Tw  w^T\\
&=\|w\|\|w^*\|\phi_0(w,w^*) (w-w^*)^T 
+ \phi_0(w,w^*)\|w^*\|\|w\|w^T  \\
&\quad 
-\phi_1(w,w^*)\|w^*\|^2 w^T 
-\phi_2(w,w^*)\|w^*\|^2 \cos \theta(w,w^*)  w^T \\   
&=2\|w\|\|w^*\|\phi_0(w,w^*) w^T-\|w\|\|w^*\|\phi_0(w,w^*) (w^*)^T 
 \\
&\quad 
-\|w^*\|^2\phi_1(w,w^*) w^T 
-\|w^*\|^2\phi_2(w,w^*) \cos \theta(w,w^*)  w^T \\   
&=:I_1+I_2+I_3+I_4.
\end{align*}
Next, we calculate
\begin{align*}
    &   \alpha(w, w^*) C\left(\frac{w}{\|w\|}, w \right)
-
\alpha(w^*, w^*) 
C\left(\frac{w}{\|w\|}, w ^*\right)\\
& = \|w^*\|\|w\| \phi_0(w, w^*)\left(
\phi_1(w,w)w+\phi_2(w,w)w
\right)\\
&\quad 
-
\|w^*\|\|w^*\| \phi_0(w^*,w^*)\left(
\phi_1(w,w^*)w^* 
+
\phi_2(w,w^*)\frac{\|w^*\|}{\|w\|}w
\right)\\
& = \frac{1}{2} \|w^*\|\|w\| \phi_0(w, w^*) w - \frac{1}{2}\|w^*\|^2  \phi_1(w,w^*)w^*  \\
&\quad -\frac{1}{2} \frac{\|w^*\|^3}{\|w\|}\phi_2(w,w^*)w\\
&=:J_1+J_2+J_3.
\end{align*}
Note that we have used the fact that $\phi_1(w,w) =\frac{1}{2}$ and $\phi_2(w,w)=0$.
From a direct computation, it follows that
\begin{align*}
    I_1J_1 & = \|w\|^4\|w^*\|^2\phi_0(w,w^*)^2,\\
    I_2J_1 & = -\frac{1}{2}\|w\|^3\|w^*\|^3 
    \phi_0(w, w^*)^2\cos \theta(w,w^*),\\
    I_3J_1 & = 
    -\frac{1}{2}\|w\|^3\|w^*\|^3 \phi_0(w, w^*) \phi_1(w, w^*) ,\\
    I_4J_1 & = -\frac{1}{2}\|w\|^3\|w^*\|^3 \phi_0(w, w^*) \phi_2(w, w^*) \cos \theta(w,w^*),\\
    I_1J_2 & = 
    -\|w\|^2\|w^*\|^4\phi_0(w,w^*)\phi_1(w,w^*)\cos\theta(w,w^*),\\
    I_2J_2 & =  \frac{1}{2}\|w\|\|w^*\|^5\phi_0(w,w^*)\phi_1(w,w^*),\\
    I_3J_2 & = \frac{1}{2}\|w\|\|w^*\|^5 \phi_1(w,w^*)^2\cos\theta(w,w^*),\\
    I_4J_2 & = \frac{1}{2}\|w\|\|w^*\|^5\phi_1(w^*,w^*)\phi_2(w,w^*)\cos^2\theta(w,w^*),
    \\
    I_1J_3 & = -\|w\|^2\|w^*\|^4\phi_0(w,w^*)\phi_2(w^*,w^*),\\
    I_2J_3 & = \frac{1}{2}\|w\|\|w^*\|^5 \phi_0(w^*,w^*)\phi_0(w,w^*)\phi_2(w,w^*)  \cos\theta(w,w^*) , \\
    I_3J_3 & = \frac{1}{2}\|w\|\|w^*\|^5 \phi_1(w^*,w^*)\phi_2(w,w^*), \\
    I_4J_3 & = \frac{1}{2} \|w\|\|w^*\|^5\phi_2(w^*,w^*)^2\cos\theta(w^*,w^*).
\end{align*}
Taking the summation over $I_iJ_j$ for $i \in \{1, 2, 3\}$ and $j \in \{1, 2, 3, 4\}$, we arrive at
\begin{align*}
 \frac{1}{2}\frac{d}{dt}\|w-w^*\|^2
 &= - \frac{1}{N}\sum_{k=1}^{N}| \mu_{k}|^2\|w\|\|w^*\|^5
\left(M_3\left(\frac{\|w\|}{\|w^*\|}\right)^3-M_2\left(\frac{\|w\|}{\|w^*\|}\right)^2-M_1\left(\frac{\|w\|}{\|w^*\|}\right)+M_0
\right),
\end{align*}
where 
\begin{align*}
M_{3} & =  \phi_0(w, w^*)^2, \\
M_{2} & =\frac{1}{2} \phi_0(w,w^*)
\Big[ \phi_0(w,w^*)\cos \theta(w,w^*) +\phi_1(w,w^*)+\phi_2(w,w^*)\cos\theta(w,w^*)
\Big], \\
M_{1} & =\phi_0(w,w^*)^2,\\
M_{0} & = \phi_0(w,w^*)^2.
\end{align*}
Then we arrive
\begin{equation}\label{eq:d/dt_1}
\begin{aligned}
    &\frac{1}{2} \frac{d}{dt}\|w-w^*\|^2\\
     &=  - \frac{1}{N}\sum_{k=1}^{N}| \mu_{k}|^2 \phi_0(w,w^*)^2 \|w\|\|w^*\|^5
     \left(\left(\frac{\|w\|}{\|w^*\|}\right)^3-\left(\frac{\|w\|}{\|w^*\|}\right)^2-\left(\frac{\|w\|}{\|w^*\|}\right)+1.
\right)\\
 & \quad  + \frac{1}{N}\sum_{k=1}^{N} \frac{| \mu_{k}|^2}{2}\|w\|^3\|w^*\|^3 \phi_0(w,w^*)
 \Big[(\cos \theta(w,w^*) - 2)\phi_0(w,w^*)+\phi_1(w,w^*)+\phi_2(w,w^*)\cos\theta(w,w^*)\Big] \\
 &=: K_1+K_2.
\end{aligned}
\end{equation}
It is directly checked that  $t^3-t^2-t+1\geq 0\quad \text{for all }t\geq 0$
which implies that  $K_1\leq 0$.
By denoting $\theta = \theta(w,w^*)$ , we compute 
\begin{align*}
&(\cos  \theta- 2)\phi_0(w,w^*)+\phi_1(w,w^*)+\phi_2(w,w^*)\cos\theta(w,w^*) \\
&=((\pi - \theta)\cos\theta+\sin \theta)(\cos \theta -2)+
(\pi-\theta) 
+
\sin \theta\cos \theta\\
&=(\cos \theta -1)((\pi-\theta)(\cos\theta -1)+2\sin \theta)=:g(\theta).
\end{align*}
By Lemma \ref{lem:g_less_0}, we have $g(\theta)\leq 0$ for all $\theta \in [0, \pi]$ implying that $K_2\leq 0$. Note that
\[
\frac{1}{2} \frac{d}{dt}\|w-w^*\|^2 = 0 \iff K_1 = K_2 = 0.
\]
This equality holds only when $\|w\| = \|w^*\|$ and $\theta(w, w^*) = 0$ or $\pi$. If $\theta(w, w^*) = 0$, then $w = w^*$. On the other hand, if $\theta(w, w^*) = \pi$ and $\|w\| = \|w^*\|$, then $w = -w^*$ and $w \not\in \{\|w-w^*\| < \|w^*\|\}$. Therefore, we have
\[
\frac{d}{dt}\|w-w^*\|^2 < 0 \quad \text{under the assumption } w_0 \in \{\|w-w^*\| < \|w^*\|\}. 
\]
Moreover, from the fact that $\frac{1}{2}\|w-w^*\|^2 \leq \frac{1}{2}\|w^*\|^2$, once $w_0 \in \{\|w-w^*\| < \|w^*\|\}$, the parameter vector $w$ remains in $\{\|w-w^*\| < \|w^*\|\}$, thus $w \to w^*$ as $t \to \infty$.

% Let us show that adding derivative information to the loss function enhances convergence.
% Taking derivative in $x$ variable and represent   \cref{eq:u_kXW} using a vector $x\in \mathbb{R}^n$ instead of the matrix $X\in \mathbb{R}^{J\times n}$, we have
%     \begin{align*}
%         D_x u_{k}(x; w) &=D_x\left[\mu_k\alpha(w, w^*) \mathbbm{1}_
%         {\{x^T w>0 \}}x^Tw)\right]\\
%     &=\mu_k\alpha(w, w^*) \mathbbm{1}_
%         {\{x^T w>0 \}}w.
%     \end{align*}
% Similarly, use the vector representation of 
Let us show that adding derivative information to the loss function enhances convergence speed. Representation \cref{eq:u_kXW}  using a vector $x \in \mathbb{R}^n$ instead of the matrix $X \in \mathbb{R}^{J \times n}$ then taking derivative in $x$ variable, we have
\begin{align*}
    D_x u_{k}(x; w) &= D_x\left[\mu_k\alpha(w, w^*) \mathbbm{1}_{\{x^T w > 0\}}(x^Tw)\right]\\
    &= \mu_k\alpha(w, w^*) \mathbbm{1}_{\{x^T w > 0\}}w.
\end{align*}
Similarly, when using the vector representation of 
   \cref{eq:Du_kXW}, we find 
   \begin{align*}
       D_xD_w  u_k(x;w) 
       & =  D_x\left[\mu_k\alpha(w, w^*) \mathbbm{1}_{x^Tw} x +  \frac{\mu_k }{\tilde{J}} \mathbbm{1}_
        {\{x^T w>0 \}} x^Tw \mathbb{E}_{Y}\left[   Y^T D(Y, w)D(Y,w^*)Yw^*\right] \right]\\
       & =  \mu_k\alpha(w, w^*) 
       \mathbbm{1}_
        {\{x^T w>0 \}}  I_d+  \frac{\mu_k }{\tilde{J}} \mathbbm{1}_
        {\{x^T w>0 \}}   \mathbb{E}_{Y}\left[   Y^T D(Y, w)D(Y,w^*)Yw^*\right] w^T\\
        & =  \mu_k\alpha(w, w^*) 
       \mathbbm{1}_
        {\{x^T w>0 \}}  I_d+  \mu_k  \mathbbm{1}_
        {\{x^T w>0 \}}   C\left(\frac{ w}{\|w\|}, w^*\right)w^T
   \end{align*}
Let $\mathcal{L}_{Der}$ be the loss function defined in equation \cref{L2_der_loss}. We then compute 
\begin{equation}\label{DwL_Der}
\begin{aligned}
&D_w \mathcal{L}_{Der} \\
&=
\frac{1}{JN} 
\sum_{k=1}^{N}\sum_{j = 1}^{J}
D_w D_xu_k(x_j, w)^T ( 
D_xu_k(x_j, w) - D_xu_k(x_j, w^*) )\\
& = \frac{1}{JN} 
\sum_{k=1}^{N}\sum_{j = 1}^{J}  |\mu_k|^2 \left(\alpha(w, w^*) 
       \mathbbm{1}_
        {\{x_j^T w>0 \}}  I_d+    \mathbbm{1}_
        {\{x_j^T w>0 \}}   w^TC\left(\frac{w}{\|w\|}, w^*\right)^T\right)\\
        &\quad \times\left(\alpha(w, w^*) \mathbbm{1}_
        {\{x_j^T w>0 \}}w - \alpha(w^* ,w^*) \mathbbm{1}_
        {\{x_j^T w^*>0 \}}w^*\right)\\
        &=\frac{1}{JN} 
\sum_{k=1}^{N}\sum_{j = 1}^{J}  |\mu_k|^2 
\left(
\alpha(w, w^*)^{2} \mathbbm{1}_{\{x_j^T w >0 \}} w +\alpha(w, w^*) \mathbbm{1}_{\{x_j^T w >0 \} }wC\left(\frac{w}{\|w\|},w^*\right)^Tw \right.\\
&\quad  \left. - \alpha(w, w^*) \alpha(w^*,w^*)\mathbbm{1}_{\{x_j^Tw>0\}}
\mathbbm{1}_{\{x_j^Tw^*>0\}}w^* - 
\alpha(w, w^*) \mathbbm{1}_{\{x_j^Tw  >0\}} \mathbbm{1}_{\{x_j^Tw^*>0\}}wC\left(\frac{w}{\|w\|}, w^* \right)^Tw^*\right).
\end{aligned}
\end{equation}
From the fact that  $x_j \sim  \mathcal{N}(0,I)$, we observe that $x_j^T w \sim  \mathcal{N}(0,\|w\|^2)$,  $x_j^T w^* \in \mathcal{N}(0,\|w^*\|^2)$, 
and the correlation between $x_j^T w$ and $x_j^T w^*$ is 
$$
\text{corr}(x_j^T w,x_j^T w^*) =\frac{\mathbb{E}_{x_j}\left[
(x_j^T w)  (x_j^T w)\right]} {\mathbb{E}\left[
(x_j^T w)^2 \right]\mathbb{E}\left[
(x_j^T w^*)^2 \right]} = \frac{w^Tw^*}{\|w\|\|w^*\|}.
$$
From these observations and Lemma \ref{lem:bivariate_normal}, it is directly checked that 
\begin{equation}\label{eq:probability}
\begin{aligned}
\mathbb{E}_{x_j} \left[\mathbbm{1}_{\{x_j^Tw >0 \}}\right] &= 
\mathbb{P}(x_j^Tw >0)=\frac{1}{2},\\
\mathbb{E}_{x_j} \left[\mathbbm{1}_{\{x_j^Tw >0 \}}\mathbbm{1}_{\{x_j^Tw^* >0 \}}\right] 
&= 
\mathbb{P}(x_j^Tw >0,x_j^Tw^* >0)=\frac{1}{4} +\frac{1}{2\pi} \arcsin
\left({\frac{w^T w^*}{\|w\|\|w^*\|}}\right)\\ 
&=\frac{1}{4} +\frac{1}{2\pi} \left(\frac{\pi}{2}  -\theta(w,w^*) \right)    
= \phi_1(w,w^*).
\end{aligned}
\end{equation}
Let  $w$ be the  parameter vector  following $\cref{eq:gradient_flow}_2$. Then it follows that
\begin{align*}
     &\frac{1}{2}\frac{d}{dt} \|w-w^*\|^2 \\& =  - (w - w^*)^ T  \mathbb{E}_{X} \left[\frac{\partial \mathcal{L}_{Sob} }{\partial w} \right]\\
     &=  - (w - w^*)^ T  \mathbb{E}_{X} \left[\frac{\partial \mathcal{L}_{L^2} }{\partial w} \right] - (w - w^*)^ T  \mathbb{E}_{X} \left[\frac{\partial \mathcal{L}_{Der} }{\partial w} \right]=: I_1+I_2.
\end{align*}
To prove the effectiveness of Sobolev Training, we are left to show that $I_2\leq 0$.
By \cref{DwL_Der,eq:probability}, we find
\begin{align*}
  I_2   &= -\frac{1}{JN} 
\sum_{k=1}^{N}\sum_{j = 1}^{J}  |\mu_k|^2  
     (w - w^*)^ T  \Bigg(
\alpha(w, w^*)^{2}\mathbb{E}_{x_j} \left[
\mathbbm{1}_{\{x_j^T w >0 \}} \right] w \\
&\quad +\alpha(w, w^*) \mathbb{E}_{x_j} \left[
\mathbbm{1}_{\{x_j^T w >0 \}} \right] wC\left(\frac{w}{\|w\|},w^*\right)^Tw \\
&\quad   - \alpha(w, w^*) \alpha(w^*,w^*)\mathbb{E}_{x_j} \left[
\mathbbm{1}_{\{x_j^Tw>0\}}
\mathbbm{1}_{\{x_j^Tw^*>0\}}\right]
w^* \\
&\quad - 
\alpha(w, w^*) 
\mathbb{E}_{x_j} \left[
\mathbbm{1}_{\{x_j^Tw>0\}}
\mathbbm{1}_{\{x_j^Tw^*>0\}}\right]wC\left(\frac{w}{\|w\|}, w^* \right)^Tw^*\Bigg)\\
&=:-\frac{1}{JN} 
\sum_{k=1}^{N}\sum_{j = 1}^{J}  |\mu_k|^2  (w-w^*)^TI_{2k}
\end{align*}
% Rewriting equality above, we have 
% \begin{align*}
%      &\frac{d}{dt} \|w-w^*\|^2 \\& =  - (w - w^*)^ T  \mathbb{E}_{X} \left[\frac{\partial \mathcal{L}_2 }{\partial w} \right]\\
%      &
% = -\frac{1}{N} 
% \sum_{k=1}^{N}   |\mu_k|^2  
%      (w - w^*)^ T  \left(
% \alpha(w, w^*)^{2}\mathbb{E}_{x_j} \left[
% \mathbbm{1}_{\{x_j^T w >0 \}} \right] w +\alpha(w, w^*) \mathbb{E}_{x_j} \left[
% \mathbbm{1}_{\{x_j^T w >0 \}} \right] wC\left(\frac{w}{\|w\|},w^*\right)^Tw \right.\\
% &\quad  \left. - \alpha(w, w^*) \alpha(w^*,w^*)\mathbb{E}_{x_j} \left[
% \mathbbm{1}_{\{x_j^Tw>0\}}
% \mathbbm{1}_{\{x_j^Tw^*>0\}}\right]
% w^* - 
% \alpha(w, w^*) 
% \mathbb{E}_{x_j} \left[
% \mathbbm{1}_{\{x_j^Tw>0\}}
% \mathbbm{1}_{\{x_j^Tw^*>0\}}\right]wC\left(\frac{w}{\|w\|}, w^* \right)^Tw^*\right)
%     \end{align*}
With help of \cref{eq:probability}, we have 
\begin{align*}
    I_{2,k} &
    = \frac{\alpha(w, w^*)^2}{2}w + \frac{\alpha(w, w^*)}{2} \left(\phi_1(w,w^*)w^*+\phi_2(w,w^*)\frac{\|w^*\|}{\|w\|}w \right)^T w  w\\
    &\quad -\alpha(w, w^*)\alpha(w^*,w^*)\phi_1(w,w^*)w^* \\
    &\quad - \alpha(w, w^*)\phi_1(w,w^*)\left(
    \phi_1(w,w^*)w^*+\phi_2(w,w^*)\frac{\|w^*\|}{\|w\|}w
    \right)^Tw^* w\\
    & = \frac{1}{2}\|w\|^2 \|w^*\|^2 \phi_0(w,w^*)^2 w\\
    &\quad + \frac{1}{2} \|w\|^2\|w^*\|^2\phi_0(w,w^*)\left(\phi_1(w,w^*)\cos\theta(w,w^*) + \phi_2(w,w^*) \right) w  \\
    &\quad -\frac{1}{2}\|w^*\|^3\|w\|\phi_0(w,w^*)\phi_1(w,w^*)w^*\\
    &\quad -\|w^*\|^3\|w\|\phi_0(w,w^*)\phi_1(w,w^*)\left(
    \phi_1(w,w^*)
    +\phi_2(w,w^*)\cos\theta(w,w^*)
    \right)w\\
    &=
    \|w\|^2\|w^*\|^2\phi_0(w,w^*)^2 w -\frac{1}{2}\|w^*\|^3\|w\|\phi_0(w,w^*)\phi_1(w,w^*)w^*\\
    &\quad -\|w^*\|^3\|w\|\phi_0(w,w^*)\phi_1(w,w^*)\left(
    \phi_1(w,w^*)
    +\phi_2(w,w^*)\cos\theta(w,w^*)
    \right)w.
\end{align*}
Then we calculate to have
\begin{align*}
    &(w-w^*)I_{2,k} \\
    &=
    \|w\|^4\|w^*\|^2\phi_0(w,w^*)^2\\
    &\quad -\|w\|^3\|w^*\|^3\phi_0(w,w^*)^2\cos\theta(w,w^*)\\
    &\quad - \frac{1}{2}\|w^*\|^4\|w\|^2\phi_0(w,w^*)\phi_1(w,w^*)\cos \theta(w,w^*)\\
    &\quad +\frac{1}{2}\|w^*\|^5\|w\|\phi_0(w,w^*)\phi_1(w,w^*)\\
    &\quad -\|w^*\|^3\|w\|^3\phi_0(w,w^*)\phi_1(w,w^*)\left(
    \phi_1(w,w^*)+\phi_2(w,w^*)\cos\theta(w,w^*)
    \right)\\
    &\quad +\|w^*\|^4\|w\|^2\phi_0(w,w^*)\phi_ 1(w,w^*)\cos\theta(w,w^*)
    \left(
    \phi_1(w,w^*)+\phi_2(w,w^*)\cos\theta(w,w^*)
    \right).
    \end{align*}
If  $\phi_0(w,w^*)\not  = 0$, we rearrange above equality as 
\begin{equation}\label{eq:d/dt_2}
\begin{aligned}
        \frac{(w-w^*)^T  
        I_{2,k}}{\phi_0(w,w^*)\|w\|\|w^*\|^{5}}
    & =   M_3 
    \left( \frac{\|w\|}{\|w^*\|}\right)^3 -M_2  \left( \frac{\|w\|}{\|w^*\|}\right)^2 -  M_1\left( \frac{\|w\|}{\|w^*\|}\right)+M_0,
\end{aligned}    
\end{equation}
where 
\begin{align*}
M_{3} & =  \phi_0(w, w^*), \\
M_{2} & =\phi_0(w,w^*)\cos\theta(w,w^*)+\phi_1(w,w^*)(\phi_1(w,w^*)+\phi_2(w,w^*)\cos\theta(w,w^*)), \\
M_{1} & =\phi_1(w,w^*)\cos\theta(w,w^*) \left(\frac{1}{2}   -
\phi_1(w,w^*)- \phi_2(w,w^*)\cos\theta (w,w^*)\right),\\
M_{0} & = \frac{1}{2}\phi_1(w,w^*).
\end{align*}
For notational convenience, we denote
$$
f(t) = M_3t^2-M_2t^2-M_1t +M_0.
$$
Since $\phi_0(w,w^*) \geq 0$, the signs of both sides of \cref{eq:d/dt_2}  are the same. By Lemma \ref{lem:sign_of_3_oder_poly}, $f(t)$ has a local minimum value at
\begin{equation*}
t_0 = \frac{M_2 +\sqrt{M_2^2+3M_1M_3}}{3M_3},
\end{equation*}
with the value of 
\begin{equation*}
f(t_0) = \frac{1}{27M_3^2}\left( 27M_3^2M_0 -2M_2^3 - 9M_1M_2M_3 - 2(M_2^2+3M_1M_3)\sqrt{M_2^2+3M_1M_3} \right).
\end{equation*}
We may assume that $M_3^2+3M_1M_3\geq 0$. Otherwise, $f(t)$ is an increasing function, and since $f(0)= M_0\geq 0$, there is nothing to show. We are left to show that $f(t_0)\geq 0$ if $M_2^2+3M_1M_3 \geq 0$. By denoting
\begin{equation}\label{eq:inequality_to_show}
   h(\theta(w,w^*)) := \frac{27}{2} \phi_0(w,w^*)^2\phi_1(w,w^*) - \left(2M_2^3 -9M_1M_2 M_3 -2(M_2^2+3M_1M_3)\sqrt{M_2^2+3M_1M_3} \right),
\end{equation}
we need to show that $h(\theta)\geq 0$ for all $\theta\in [0,\pi]$ provided that $M_2^2 + 3 M_1M_3\geq 0$. Since showing that $h(\theta)\geq 0$ analytically is excessively technical and falls outside the scope of this work, we shall assume that $h(\theta)\geq 0$ holds for all $0\leq \theta \leq \pi$ to complete the proof.
We refer to Figure \ref{fig:h_bound} in Appendix \ref{appendix:visual_loss} for the numerical evidence.

Let us prove the strict inequality \cref{eq:strict_inequality}. 
Note that $h(\theta)$ has property that
$$
h(\theta)= 0 \iff \theta=0.
$$
Therefore if $\theta >0$ and  
$(w-w^*)^TI_{2,k}=0$ then $\phi_0(w,w^*)=0$ implying that $\theta = \pi$.
Therefore, we conclude that 
$$
\frac{d}{dt}\|w_{Sob}-w^*\|^2 < \frac{d}{dt}\|w_{L^2}-w^*\|^2 <0
\quad
\text{for all $0<\theta<\pi$.}
$$
\end{proof}

\section{Visualization of the loss functions}\label{appendix:visual_loss}
This section provides a visual guide of a 3D visual representation of a loss function landscape and shows that Sobolev Training is effective. For simplicity, we assume $N=1$ and $\mu_1=1$ in 
 \cref{eq:d/dt_1,eq:d/dt_2}. 
Then we denote 
$$V_{L^2}  :=  \frac{\frac{d}{dt}\|w_{L^2}-w^*\|^2}{2 \phi_0(w,w^*) \|w\|\|w^*\|^5} 
\quad \text{and}\quad
V_{Sob}  :=  \frac{\frac{d}{dt}\|w_{Sob}-w^*\|^2}{2 \phi_0(w,w^*) \|w\|\|w^*\|^5} .
$$
Then we could obtain the following loss function landscape for $0\leq \theta\leq \pi$ and $x= \frac{\|w\|}{\|w^*\|}$ in Figure \ref{fig:loss_landscape_3d}. 
 We also provide a comparison for $V_{L^2}$  and 
$V_{Sob}$ for fixing one variable  in Figure \ref{fig:loss_landscape_2d_x_fix} and Figure \ref{fig:loss_landscape_2d_theta_fix}.
Moreover, since proving the bound \cref{eq:inequality_to_show} is technically challenging, we provide a visual graph of the function $h(x)$ in Figure \ref{fig:h_bound}.

%%%%%%%%%%%%%%%%%%%%%%%%%%%%%%%%%%%%%%%%%%%%%%%%%%%%%%%%%
\begin{figure}[t]
\vskip 0.2in
\begin{center}
\centerline{\includegraphics[width=1.0\columnwidth]{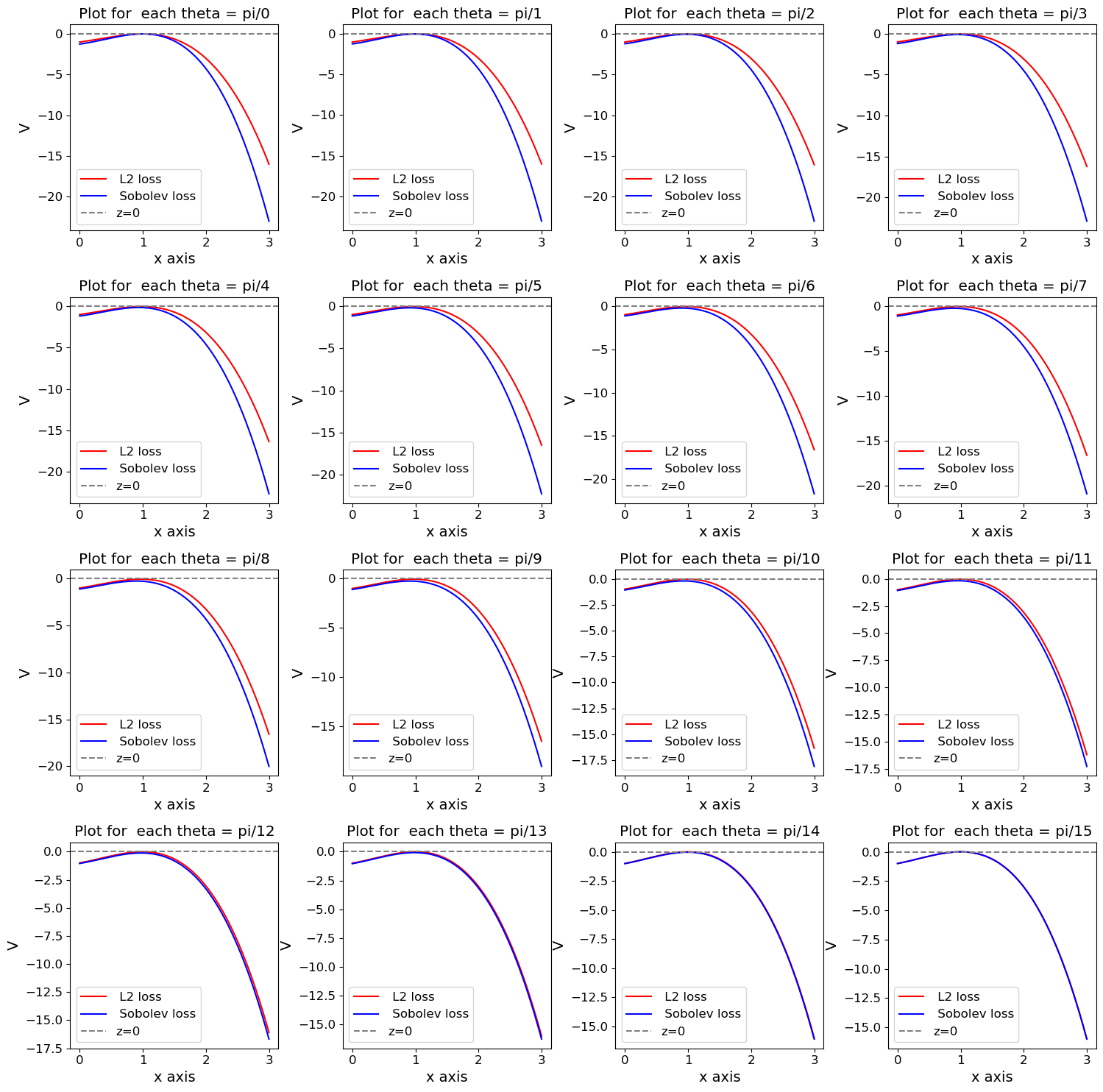}}
\caption{Visualization of comparison between loss function landscape for fixed $\theta\in [0, \pi]$.
}
\label{fig:loss_landscape_2d_x_fix}
\end{center}
\vskip -0.2in
\end{figure}
%%%%%%%%%%%%%%%%%%%%%%%%%%%%%%%%%%%%%%%%%%%%%%%%%%%%%%%%%%%
\begin{figure}[t]
\vskip 0.2in
\begin{center}
\centerline{\includegraphics[width=1.0\columnwidth]{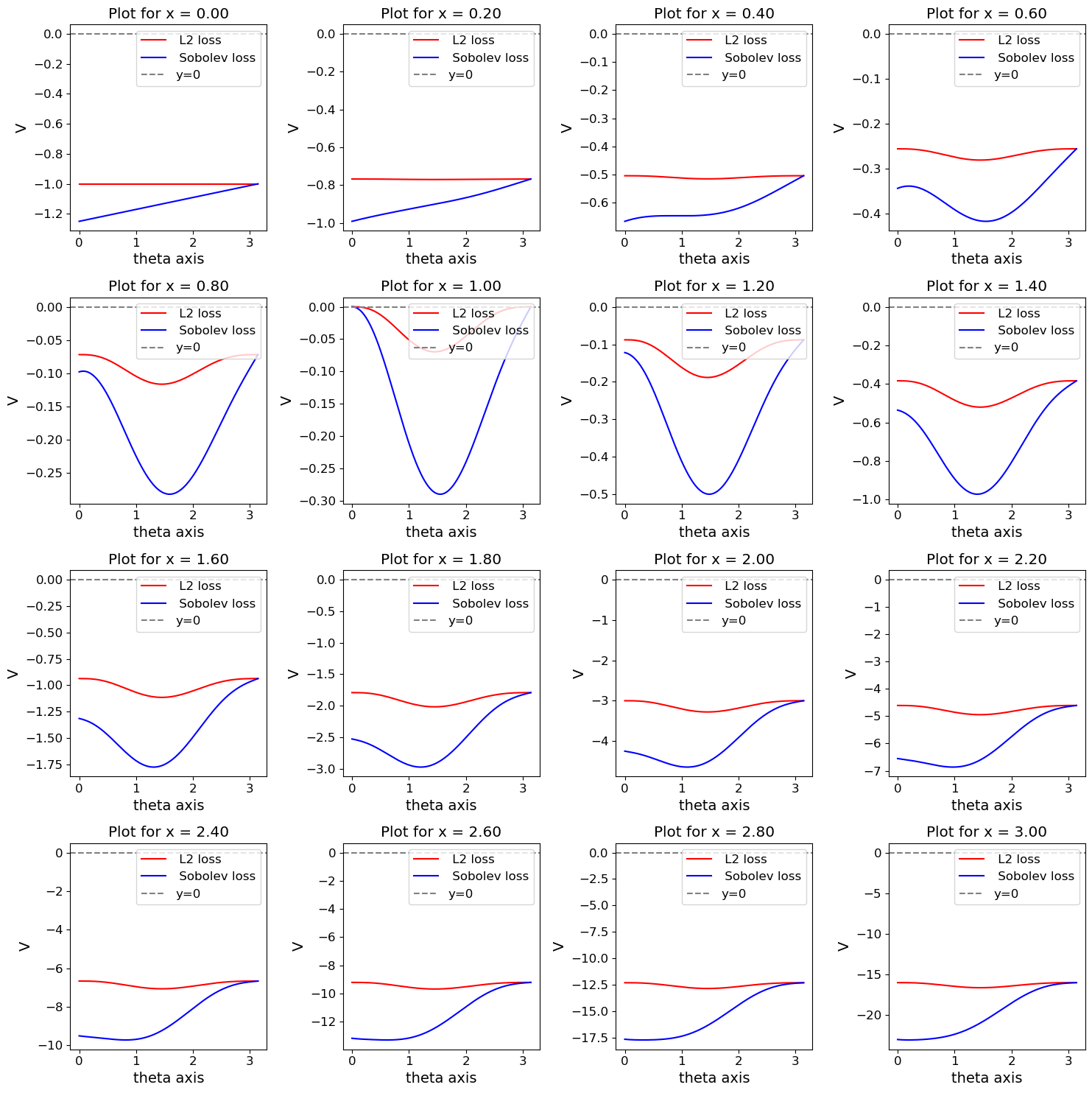}}
\caption{Visualization of comparison between loss function landscape for fixed  $x\in [0, 3]$.
}
\label{fig:loss_landscape_2d_theta_fix}
\end{center}
\vskip -0.2in
\end{figure}
%%%%%%%%%%%%%%%%%%%%%%%%%%%%%%%%%%%%%%%%%%%%%%%%%%%%%%%%%%%
\begin{figure}[t]
\vskip 0.2in
\begin{center}
\centerline{\includegraphics[width=1.0\columnwidth]{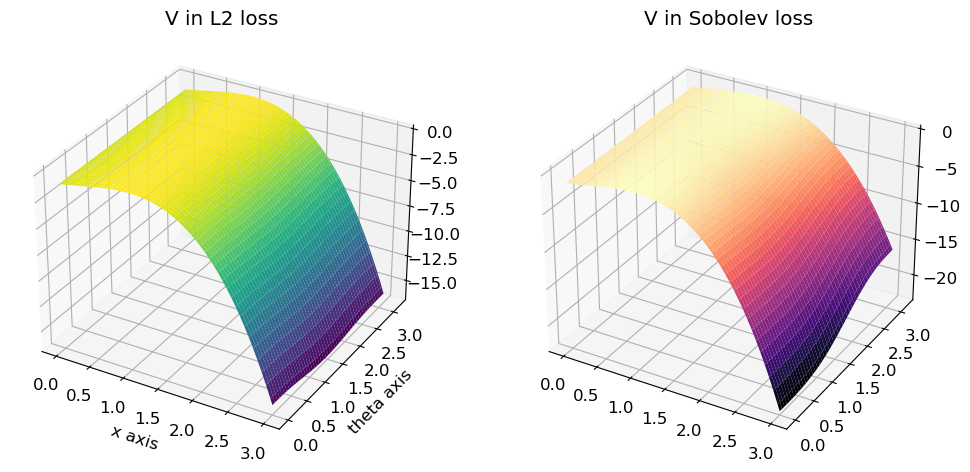}}
\caption{Visualization of loss function landscape for $L^2$ loss and Sobolev loss in 3D.
}
\label{fig:loss_landscape_3d}
\end{center}
\vskip -0.2in
\end{figure}

\begin{figure}
\begin{center}
\centerline{\includegraphics[width=1.0\columnwidth]{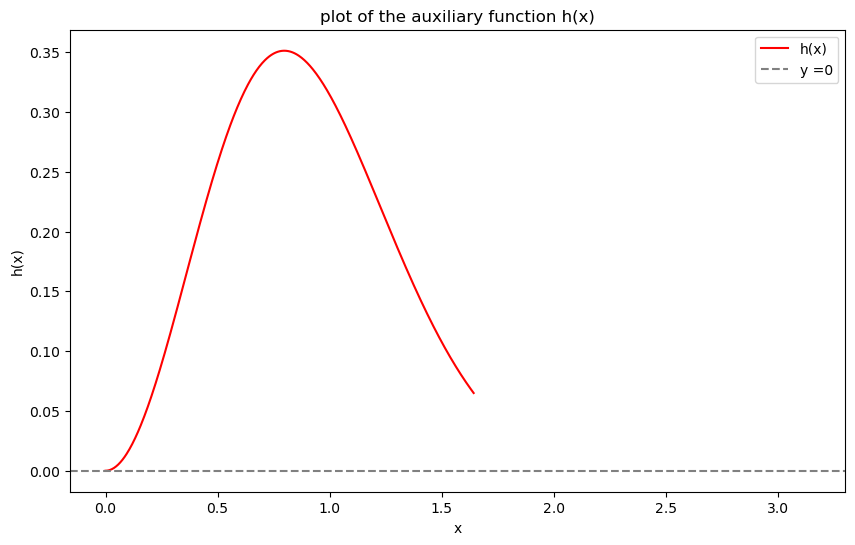}}
\caption{Visualization of $h(\theta)$ define in \cref{eq:inequality_to_show}. Note the $h(\theta)$ is not defined for all $\theta\in [0,\pi]$ because we need to exclude the case when $M_2^2+3M_1M_3<0$.
}
\label{fig:h_bound}
\end{center}
\vskip -0.2in
\end{figure}

\end{document}